\documentclass[11pt, letterpaper]{article} 

\usepackage{thmtools}

\usepackage{amsthm}
\def\showauthornotes{0}

\def\showkeys{0}
\def\showdraftbox{0}
\def\showcolorlinks{1}
\def\usemicrotype{1}
\def\showfixme{0}

\newcommand{\algname}{}
\newtheorem{theorem}{Theorem}
\newtheorem{definition}{Definition}

\newcommand{\DeclareUrlCommand}[1]{}
\usepackage{amsfonts,amssymb,amsmath}
\usepackage{xspace}
\usepackage{thm-restate}
\usepackage{smartref}

\newcommand{\ellzo}{\ell_{0,1}}
\newcommand{\secref}[1]{Sec. \ref{#1}}
\newcommand{\lemref}[1]{Lemma \ref{#1}}
\newcommand{\xx}{\mathbf{x}}
\newcommand{\yy}{\mathbf{y}}
\newcommand{\ww}{\mathbf{w}}
\newcommand{\X}{\mathcal{X}}
\renewcommand{\P}{\mathbb{P}}
\newcommand{\Y}{\mathcal{Y}}
\renewcommand{\H}{\mathcal{H}}
\newcommand{\vv}{\mathbf{v}}

\newcommand{\EE}[1]{\mathbb{E}\left[#1\right]}
\newcommand{\loss}{\mathcal{L}}
\newcommand{\jvec}{\mathbf{n}}
\newcommand{\johnson}{\mathcal{H_J}}
\newcommand{\rad}{\mathcal{R}}
\newcommand{\con}{\mathcal{C}}
\newcommand{\uu}{\mathbf{u}}
\newcommand{\conj}{C_{\wedge}}
\newcommand{\ignore}[1]{}
\newcommand{\thmref}[1]{Thm. \ref{#1}}

\newcommand{\dotp}[2]{{\langle #1, #2 \rangle}}
\newcommand{\sC}{\mathcal{J}^s}
\newcommand{\roi}[1]{\color{red}Roi:#1\color{black}}
\newcommand{\universal}{U}
\newcommand{\regular}{Euclidean\xspace}

\usepackage{xspace,enumerate}

\usepackage[american]{babel}

\usepackage{mathtools}

\newcommand{\full}[2]{\if01
#1 
\else
 #2
 \fi}

\newtheorem{fact}{Fact}
\newcommand{\kp}{\bar{k}}

\newtheorem{question}[theorem]{Question}

\usepackage{newpxtext} 
\usepackage{textcomp} 
\usepackage[varg,bigdelims]{newpxmath}
\usepackage[scr=rsfso]{mathalfa}
\usepackage{bm} 
\let\mathbb\varmathbb

\ifnum\showkeys=1
\usepackage[color]{showkeys}
\fi

\ifnum\showcolorlinks=0
\usepackage[
pagebackref,
colorlinks=false,
pdfborder={0 0 0}
]{hyperref}
\fi

\newcommand{\Sref}[1]{\hyperref[#1]{\S\ref*{#1}}}
\newcommand{\s}[2]{\mathbf{#1}^{(#2)}}
\usepackage{nicefrac}

\ifnum\usemicrotype=1
\usepackage{microtype}
\fi

\ifnum\showauthornotes=1
\newcommand{\Authornote}[2]{{\sffamily\small\color{red}{[#1: #2]}}}
\newcommand{\Authornotecolored}[3]{{\sffamily\small\color{#1}{[#2: #3]}}}
\newcommand{\Authorcomment}[2]{{\sffamily\small\color{gray}{[#1: #2]}}}
\newcommand{\Authorstartcomment}[1]{\sffamily\small\color{gray}[#1: }

\newcommand{\Authorfnote}[2]{\footnote{\color{red}{#1: #2}}}
\newcommand{\Authorfixme}[1]{\Authornote{#1}{\textbf{??}}}
\newcommand{\Authormarginmark}[1]{\marginpar{\textcolor{red}{\fbox{\Large #1:!}}}}
\else
\newcommand{\Authornote}[2]{}
\newcommand{\Authornotecolored}[3]{}
\newcommand{\Authorcomment}[2]{}
\newcommand{\Authorstartcomment}[1]{}

\newcommand{\Authorfnote}[2]{}
\newcommand{\Authorfixme}[1]{}
\newcommand{\Authormarginmark}[1]{}
\fi

\newcommand{\Pnote}{\Authornote{P}}

\newcommand{\Rnote}{\Authornote{M}}

\ifnum\showfixme=0

\fi

\usepackage{boxedminipage}

\newcommand{\Esymb}{\mathbb{E}}
\newcommand{\Psymb}{\mathbb{P}}

\DeclareMathOperator*{\E}{\Esymb}

\DeclareMathOperator*{\ProbOp}{\Psymb}

\renewcommand{\Pr}{\ProbOp}

\newcommand{\textparen}[1]{\text{(#1)}}

\ifx\because\undefined
\newcommand{\because}[1]{\textparen{because #1}}
\else
\renewcommand{\because}[1]{\textparen{because #1}}
\fi

\newcommand\bdot\bullet

\ifx\mathds\undefined 

\else
}
\fi

\DeclareMathOperator{\val}{val}

\DeclareMathOperator{\opt}{opt}

\DeclareMathOperator{\poly}{poly}

\newcommand{\R}{\mathbb R}
\newcommand{\C}{\mathbb C}

\newcommand{\cB}{\mathcal B}

\newcommand{\cD}{\mathcal D}

\newcommand{\cF}{\mathcal F}

\newcommand{\cJ}{\mathcal J}

\newcommand{\cL}{\mathcal L}

\newcommand{\bbB}{\mathbb B}
\newcommand{\bbS}{\mathbb S}

\renewcommand{\leq}{\leqslant}
\renewcommand{\le}{\leqslant}
\renewcommand{\geq}{\geqslant}
\renewcommand{\ge}{\geqslant}

\ifnum\showdraftbox=1

\else

\fi

\let\epsilon=\varepsilon

\numberwithin{equation}{section}

\newcommand\MYcurrentlabel{xxx}

\newcommand{\MYstore}[2]{%
  \global\expandafter \def \csname MYMEMORY #1 \endcsname{#2}%
}

\newcommand{\MYload}[1]{%
  \csname MYMEMORY #1 \endcsname%
}

\newcommand{\MYnewlabel}[1]{%
  \renewcommand\MYcurrentlabel{#1}%
  \MYoldlabel{#1}%
}

\newcommand{\MYdummylabel}[1]{}

\newcommand{\torestate}[1]{%
  \let\MYoldlabel\label%
  \let\label\MYnewlabel%
  #1%
  \MYstore{\MYcurrentlabel}{#1}%
  \let\label\MYoldlabel%
}

\newcommand{\restatetheorem}[1]{%
  \let\MYoldlabel\label
  \let\label\MYdummylabel
  \begin{theorem*}[Restatement of \prettyref{#1}]
    \MYload{#1}
  \end{theorem*}
  \let\label\MYoldlabel
}

\newcommand{\restatelemma}[1]{%
  \let\MYoldlabel\label
  \let\label\MYdummylabel
  \begin{lemma*}[Restatement of \prettyref{#1}]
    \MYload{#1}
  \end{lemma*}
  \let\label\MYoldlabel
}

\newcommand{\restateprop}[1]{%
  \let\MYoldlabel\label
  \let\label\MYdummylabel
  \begin{proposition*}[Restatement of \prettyref{#1}]
    \MYload{#1}
  \end{proposition*}
  \let\label\MYoldlabel
}

\newcommand{\restatefact}[1]{%
  \let\MYoldlabel\label
  \let\label\MYdummylabel
  \begin{fact*}[Restatement of \prettyref{#1}]
    \MYload{#1}
  \end{fact*}
  \let\label\MYoldlabel
}

\newcommand{\restate}[1]{%
  \let\MYoldlabel\label
  \let\label\MYdummylabel
  \MYload{#1}
  \let\label\MYoldlabel
}

\let\origparagraph\paragraph
\renewcommand{\paragraph}[1]{\origparagraph{#1.}}

\allowdisplaybreaks

\usepackage{paralist}

\usepackage{comment}

\usepackage{braket}

\DeclareUrlCommand\email{}

\DeclareMathOperator{\zo}{\{0,1\}}

\renewcommand{\full}[2]{1}
\usepackage{times}
\usepackage{aliascnt}
\newtheorem{lem}{Lemma}
\usepackage{natbib}
\title{On the Expressive Power of Kernel Methods \\ and the Efficiency of Kernel Learning by Association Schemes}
\author{%
Pravesh K. Kothari \thanks{Princeton University and IAS \texttt{kothari@cs.princeton.edu }.}
\and
Roi Livni
\thanks{Tel Aviv University \texttt{RLivni@tauex.tau.ac.il}.}
}

\usepackage[capitalize]{cleveref}
\newtheorem{corollary}{Corollary}
\newtheorem{remark}{Remark}
\Crefname{lem}{Lemma}{Lemmas}
\Crefname{fact}{Fact}{Facts}
\Crefname{theorem}{Theorem}{Theorems}
\newcommand{\Cs}{\C}
\renewcommand{\C}{\mathcal{J}}
\renewcommand{\H}{\mathcal{H}}
\begin{document}

\maketitle
\begin{abstract}

We study the expressive power of kernel methods and the algorithmic feasibility of multiple kernel learning for a special rich class of kernels. 

Specifically, we define \emph{\regular kernels}, a diverse class that includes most, if not all, families of kernels studied in literature such as polynomial kernels and radial basis functions. We then describe the geometric and spectral structure of this family of kernels over the hypercube (and to some extent for any compact domain).  Our structural results allow us to prove meaningfull limitations on the expressive power of the class as well as derive several efficient algorithms for learning kernels over different domains.

\end{abstract}


\section{Introduction}

Kernel methods have been a focal point of research in both theory and practice of machine learning yielding fast, practical, non-linear and easy to implement algorithms for a plethora of important problems 
\citep{cortes1995support, mika1998kernel, yang2002kernel, shalev2011learning, hazan2015classification}. 

Kernels allow learning highly non linear target functions by first embedding the domain $\X$ into a high dimensional Hilbert space via an embedding $\phi:\X\rightarrow \H$ and then learning a linear classifier in the ambient Hilbert space. Ultimately the procedure outputs a classifier of the form $x\to \langle \ww,\phi(x)\rangle$, where $\phi$ captures the non-linearities and $\ww\in \H$ is a linear classifier to be learnt.

 The power of the method arises from the fact that while $\H$ could be high or even infinite dimensional, the task can be performed efficiently so long as \textbf{a)} We are given access to an efficiently computable \emph{kernel function} $k$ such that $k(x,y) = \langle \phi(x), \phi(y)\rangle$ and \textbf{b)} The \emph{large margin assumption holds}: Namely, we assume a bound on the norm of the classifier to be learnt. Then, classical results for kernel methods imply an efficient learning algorithm in terms of the dimension and margin.

 This opens the crucial question of designing kernels and constructing an RKHS for a given task so that the large-margin assumption holds. While there's a large body of work that gives a prescription for a good kernel in various learning settings \citep{shalev2011learning, kowalczyk2001kernel, sadohara2001learning, hazan2015classification, heinemann2016improper, cho2009kernel}, the task of choosing a kernel for the application at hand typically involves creative choice and guesswork.

A natural extension of kernel methods is then by allowing \emph{Multiple Kernel Learning} (MKL). In MKL, instead of fixing a kernel, we automatically learn not only the classifier but also the embedding or kernel function.

\ignore{This leads to the central question that we seek to address in this paper.
 \begin{question}
Given i.i.d. samples from an (unknown) data distribution $\cD$ over $\X_n \times \Y$, can we efficiently simultaneously learn an \emph{optimal} RKHS embedding $(H,\phi)$ for $\X_n$ and an optimal linear classifier for $\cD$ in $H$?
\end{question}}
In general, learning an optimal kernel for specific task can be ill-posed. For e.g., given a binary classification task, an optimal kernel is given by the one-dimensional embedding $\xx \rightarrow f(\xx)$ where $f$ is the unknown Bayes optimal hypothesis. Thus, without further qualifications, the task of learning an optimal kernel is equivalent to the task of learning an arbitrary Boolean function. A natural compromise then is to find an optimal kernel (or equivalently, an RKHS embedding) from within some rich enough class of kernels. 

In this work we consider a class of kernels that contain most, if not all, explicit kernels used in practice that satisfy a simple property and we term them \emph{\regular} kernels. We deter a rigorous definition to later sections, but in a nutshell, a kernel is \regular if it depends on the scalar product and the norm of its input. The class of \regular kernels capture almost all the instances of kernels considered in prior works (see, for instance \cite{scholkopf2001learning}. For example, polynomial kernels, Gaussian kernels along with Laplacian, Exponential and Sobolev space kernels (and all of their sums and products) are \regular.

As a class, the family of functions that can be expressed in a \regular kernel space, is a highly expressive and powerful class. Indeed these include, in particular, all polynomials and can thus approximate any target function to arbitrary close precision. However, standard generalization bounds and learning guarantees rely on the large margin assumption. Thus, the objective of this work is to analyze the class of functions that can be expressed through \regular kernels under norm constraints.

The main result of this paper shows that the class of \emph{all} such large margin linear classifiers, over the hyper cube, is learnable. In fact it can be expressed using a single specific \regular kernel up to some scalable deterioration in the margin. Namely, there exists a universal \regular kernel such that any classifier in an arbitrary \regular kernel belongs to the Hilbert space defined by the universal kernel, with perhaps a slightly larger norm. As a corollary we obtain both a simple and efficient algorithm to learn the class of all \regular kernels, as well as a useful characterization of the expressive power of \regular kernels which are often used in practice. 

These results are then extended in two ways. First, we extend the result from the hypercube and show that, under certain further mild restrictions on the kernels, the results can be generalized from the hypercube to arbitrary compact domains in $\mathbb{R}^n$. Second, we also show that using convex relaxations and methods from MKL introduced in \cite{lanckriet2004learning, cortes2010generalization} one can improve the statistical sample complexity and achieve tighter generalization bounds in terms of the dimension.


Our main technical method for learning optimal \regular kernels is derived from our new characterization of the spectral structure of \regular kernels. Key to this characterization are classical results describing the spectrum of matrices of Johnson Association Scheme studied in algebraic combinatorics. Our proofs, given this connection to association schemes, are short and simple and we consider it as a feature of this work. In retrospect, the use of association schemes seems natural in studying kernels and we consider this the main technical contribution of this paper.

Studying \regular kernels over the hypercube may seem restrictive, as these kernels are often applied on real input features. However, as we next summarize, this course of study leads to important insights on the applicability of kernel methods:

First, these results can be extended to real inputs under some mild restrictions over the kernels to be learnt (namely, Lipschitness and no dependence on the norm of the input). Moreover, we believe that the technical tools we develop here,  that is -- analyzing the spectral structure of the kernel family through tools from Association Scheme and Algebraic Combinatorics, are potentially powerful for any further study of MKL in various domains. 

Second, characterizing the efficiency of kernel learning also allows us to better understand the expressive power of kernel methods. Our efficient algorithm that learns the class of \regular kernels rules out the possibility of a general reduction from learning to the design of a \regular kernel (as is possible, for example, in the more general case of arbitrary kernels). Thus, we obtain that \regular kernels with large margin cannot express intersection of halfspaces, deep neural networks etc... Currently, hardness results demonstrate limitations for each fixed kernels, and they also demonstrate that constructing or choosing a kernel might be in general hard. In contrast, our result demonstrate lack of expressive power. Namely, that for \regular kernels, hardness stems not from the design of the kernel but from a deficiency in expressivness.

Moreover, as a technical contribution, our results allow an immediate transfer of lower bounds from a single fixed kernel, to a joint uniform lower bound over the whole class of \regular kernels.  As an example we consider the problem of learning conjunctions over the hypercube -- Building upon the work of \cite{klivans2007lower}, we can show that using a single fixed kernel one cannot improve over state of the art results for agnostic learning of conjunctions. The existence of a universal kernel immediately imply that these results are true even if we allow the learner to choose the kernel in a task specific manner. Thus kernel methods, equipped with Multiple Kernel Learning techniques are still not powerful to achieve any improvement over state of the art results as long as we are restricted to \regular kernels.
\subsection{Related Work}
Kernel methods have been widely used for supervised machine learning tasks beginning with the early works of \cite{aizerman1964theoretical, boser1992training} and later in the context of support vector machines \cite{cortes1995support}. Several authors have suggested new specially designed kernels (in fact \regular kernels) for multiple learning tasks. For example, learning Boolean function classes such as DNFs, and decision trees \cite{sadohara2001learning, kowalczyk2001kernel}.  Also, several recent papers suggested and designed new \regular kernels in an attempt to mimic the computation in large, multilayer networks \cite{cho2009kernel, heinemann2016improper}. 

 Limitations on the success of kernel methods and embeddings in linear half spaces have also been studied. For specific kernels, \cite{khardon2005maximum}, as well as more general results \cite{warmuth2005leaving, ben2002limitations}. The limitations for kernel methods we are concerned with aim to capture \emph{kernel learning}, where the the kernel is distribution dependent.

Beginning with the work of \cite{lanckriet2004learning}, the problem of efficiently learning a kernel has been investigated within the framework of \emph{Multiple Kernel Learning} (MKL), where various papers have been concerned with obtaining generalization bounds (\cite{srebro2006learning, cortes2010generalization, ying2009generalization}) as well as fast algorithms. (e.g. \cite{sonnenburg2006large, kloft2008non, kloft2011lp, rakotomamonjy2008simplemkl}). Approaches beyond learning positive sums of base kernels include \emph{centered alignment} \cite{cortes2010two, cortes2012algorithms}) and some non-linear methods \cite{bach2009exploring, cortes2009learning}.


In contrast with most existing work, the class we study (\regular kernels) is not explictly described as a  non-negative sum of finite base kernels and instead it is defined by properties shared by the existing explicit kernels proposed in literature. Applied directly to learning \regular kernels, the framework of Lanckriet et al. will lead to solving an SDP of exponential size in the underlying dimension.

\section{Problem Setup and Notations}
We recall the standard setting for learning with respect to arbitrary convex loss functions. 
We consider a concept class $\cF$ to be learned over a bounded domain $\X$. In general, we will be concerned with either the hypercube $\X_n=\{0,1\}^n$, or the positive unit cube $\bbB_n = [0,1]^n \subseteq \R^n$. We will also work with individual layers of the hypercube and denote by $S_{p,n}$, the $p$-th layer of the hypercube i.e. $S_{p,n} = \{\xx\in \{0,1\}^n: \sum \xx_i =p\}$.

Given a loss function $\ell$, a distribution $\cD$ over example-label pairs from $\X \times \Y$, samples $S= \{(\s{x}{i},y_i)\}_{i \leq m}$ and any hypothesis $f$, we denote by 
\begin{align*}
\cL_{\cD}(f) = \E_{(\xx,y) \sim \cD}[ \ell(f(\xx),y)]\quad &\quad \cL_S(f) = \frac{1}{m} \sum_{i\leq m} [ \ell(f(\s{x}{i}),y_i)]
\end{align*}
the \emph{generalization error} of $f$ and the empirical error of $f$ respectively. Similarly, we set
$
\opt(\cF) := \inf_{f \in \cF} \cL_{\cD}(f)$, and $\opt_{S}(\cF) = \inf_{f \in \cF} \cL_{S}(f)
$
for the optimal error on the distribution and on the sample, respectively, of the hypothesis class $\cF$.

For convex losses, we will make  the standard assumption that $\ell$ is $L$-Lipschitz w.r.t its first argument, and we will assume that $\ell$ is bounded by $1$ at $0$, namely $|\ell(0,y)|<1$. Given a distribution $\cD$ over example-label pairs $\X \times \Y$, the algorithm's objective is to return a hypothesis $h$ such that $\cL_{\cD}(h) \leq \opt_{\cD}(\H) + \epsilon$ with probability at least $2/3$ (the confidence can be boosted in standard ways, but we prefer not to carry extra notation.) 

\paragraph{\regular RKHS Embeddings}
Our main result is an efficient algorithm for learning a \regular RKHS embedding and a linear classifier in the associated Hilbert space. 
\begin{definition}[\regular Kernel]
A kernel function $k:\X \times \X \to \mathbb{R}$ is said to be \regular if $k$ depends solely on the norms of the input and their sclar product. Namely, there exists a function $g:\mathbb{R}^3 \to \mathbb{R}$ such that
\[k(\s{x}{1},\s{x}{2})= g\left(\|\s{x}{1}\|,\|\s{x}{2}\|,\dotp{\s{x}{1}}{\s{x}{2}}\right),\]
and for all $\xx\in \X$ we assume that $k(\xx,\xx)\le 1$.
\end{definition}

We expand on the definition of \regular kernels and define \regular RKHS.

\begin{definition}[\regular RKHS]
For a Hilbert space $H$ and an embedding $\phi:\X \rightarrow H$, we say that $(H,\phi)$ is a \regular RKHS if the associated kernel function $k$ is \regular. For a fixed domain $\X$, we denote the set of all \regular RKHS for $\X$ by $\johnson(\X).$
\end{definition}
Given a Hilbert space $H$ we will also denote by $H(B)=\{\ww\in H \mid \|\ww\|_{H}\le B\}$. 
Finally, we define the class which is our focus of interest. This is the class of linear separators in \regular RKHS with a margin bound.
\begin{definition}[The Class $\C(B)$: \regular Linear Separators with a Margin]\label{def:regclass}
Fix the domain $\X$. The class of \emph{\regular linear separators} with margin $B$ is defined as the set of all linear functions in \emph{any} \regular RKHS with norm at most $B$:
\begin{align*}
\C(\X;B) = \{f_{H,\ww}:\X \rightarrow \R \mid H\in \johnson(\X),~ \ww\in H(B)\}
\end{align*}
where $f_{H,\ww}(\xx) = \dotp{\ww}{\phi(\xx)}_H.$ 
\end{definition}
For brevity of notation we will denote $\C_n(B)=\C(\X_n,B)$, and $\C_{p,n}(B)= \C(S_{p,n},B)$, and similarly $H_{\C_n}$ and $\H_{\C_{p,n}}$. 

Another class that will be technically useful in our proofs consists of all \regular kernels that can be written as direct sum of kernels over the hypercube layers:
\begin{definition}[The class $\H_{\C_{\oplus_n}}$]
The class $\H_{\C_{\oplus_n}}\subseteq \H_{\C_n}$ consists of all \regular kernels over the hypercube that are associated with RKHS $(H,\phi)$ such that $H=H_1\oplus H_2\oplus \cdots \oplus H_n$, where each $H_p$ is an RKHS with embedding $\phi_p$ such that $(H_p,\phi_p) \in \H_{\C_{p,n}}$ and such that for every $p=1,\ldots, n$: \[\phi(\xx) = (0,0,\ldots,\underbrace{\phi_p(\xx)}_{\mathrm{p^{th}~coordinate}},0,0,\ldots, 0) ,~\forall \xx\in S_{p,n},\]
Similarly we define $\C_{\oplus_n}(B)= \{f_{H,\ww} :\X_n \to \mathbb{R} \mid H\in \H_{\C_{\oplus_n}},~ \ww\in H(B)\}$.
\end{definition}
\section{Main Results}

We are now ready to state our main results. Our first result is concrened with the case that the domain is $\X_n$, the $n$-dimensional hypercube. We then proceed to improve on this result and give an analogue statment for $\bbB_n$, improve sample complexity in terms of dimension and derive limitations for kernel methods. \full{The proof of \cref{thm:main} is given in \cref{prf:main}}.

\begin{restatable}{theorem}{main}\label{thm:main}
Let $\X_n=\{0,1\}^n$ denote the $n$-th hypercube. The class of Euclidean Linear separators with a margin is learnable.

Fomally, for every $B\ge 0$ the class $\C_n(B)$ is efficiently learnable over $\zo^n$ w.r.t. any convex $L$-Lipschitz loss function $\ell$ with sample complexity $O\left (L\frac{n^3 B^2}{\epsilon^2}\right)$.

In fact, there exists a universal \regular RKHS $\universal_n$, with an efficiently computable associated kernel $k$ such that \[\C_n(B) \subseteq \universal(n^{3/2} B).\]

the kernel $k$ may be computed using a preprocess procedure with complexity $O(n^4)$, then querying at each iteration the value $k(\s{x}{i},\s{x}{j})$ for every $\s{x}{i},\s{x}{j}\in \X_n$ takes linear time in $n$.
\end{restatable}
\subsection{Corollaries and Improvements}

\subsubsection{Improving Sample Complexity through MKL}
\cref{thm:main} suggests an efficient algorithm for learning the class $\C_n(B)$ through the output of a classifier from a universal Hilbert space $\universal_n$. Since $\universal_n$ need not be the optimal Hilbert space (in terms of margin) the result may lead to suboptimal guarantees. 

One natural direction to improve over our result is by optimizing over the kernel of choice, as is done in the framework of MKL. In the next result, we follow the footsteps of \cite{lanckriet2004learning} and describe an algorithm that performs kernel learning, and we achieve improvement in terms of the dependency of the sample complexity in the dimension. On the other hand, the involved optimization task lead to some deterioration in the efficiency of the algorithm and dependence on accuracy. \full{Proof for \cref{thm:mkl} is provided in \cref{prf:mkl}}{Proofs are provided in the full version}.

\begin{restatable}{theorem}{mkl}\label{thm:mkl}
Let $\X_n=\{0,1\}^n$ denote the $n$-th hypercube. For every $B\ge 0$ the class $\C_n(B)$ is efficiently learnable over $\zo^n$ w.r.t. any convex $L$-Lipschitz loss function $\ell$ that is bounded by $1$ at zero (i.e. $|\ell(0,y)|<1$), with sample complexity given by $O\left (L\frac{n B^2}{\epsilon^3}\log n\right)$.
\end{restatable}

\subsubsection{Learning over real input features}
\cref{thm:main} shows that we can learn a \regular kernel over the domain $\X_n = \zo^n$. Kernel methods are often used in practice over real input features, therefore we give a certain extension of the aforementioned result to real input features domain. For this we need to futher restrict the family of kernels we allow to learn:

\begin{definition}[Strongly \regular Kernels]
A \regular kernel $k$ that is a kernel over $\bbB_n$ for any $n\ge 1$, is said to be $L$-Strongly \regular if $k(\s{x}{1},\s{x}{2})$ can be written as: 
\begin{align}\label{eq:sregular}
k(\s{x}{1},\s{x}{2})= g(\dotp{\s{x}{1}}{\s{x}{2}})\end{align} 
and $g$ is $L$-Lipschitz over the domain $[0,n]$.
\end{definition}

Polynomial kernels (normalized) are an example for $1$-Strongly \regular kernels, Of course also exponential kernels and other proposed kernels that have been found useful in theory (\cite{shalev2011learning}) are captured by this definition. Analogue to Definition \ref{def:regclass} we define the class of strongly \regular separators with margin and denote them by $\sC(\X, B)$.

Our next result state that analogously to the hypercube we can learn strongly \regular kernels over a compact domain. \full{This is done through discretization and reduction to the hypercube case. A proof is provided in \cref{prf:solid}.}{Again, proof is provided in the full version}

\begin{restatable}{theorem}{solid}\label{thm:solid}
 For every $B\ge 0$ the class $\sC(\bbB_n, B)$ is efficiently learnable w.r.t. any convex $L$-Lipschitz loss function, bounded by $1$ at $0$ (i.e. $|\ell(0,y)|<1$).

\end{restatable}

\subsubsection{Limitations on the expressive power of \regular kernels}
In this section we derive lower bounds for the expressive power of kernel methods. We consider as a test bed for our result the problem of agnostic conjunction learning. Arguably the simplest special case of the problem of agnostic learning  halfspaces, is captured by the task of agnostically learning conjunctions. The state of the art algorithm for agnostic learning of conjunctions over arbitrary distributions over the hypercube is based on the work of \cite{paturi1992degree} who showed that for every conjunction (equivalently, disjunctions) over the Boolean hypercube in $n$ dimensions, there is a polynomial of degree $\tilde{O}(\sqrt{n} \log{(1/\epsilon)})$ that approximates the conjunction everywhere within an error of at most $\epsilon.$ Combined with the $\ell_1$-regression algorithm of \cite{kalai2008agnostically}, this yields a $2^{\tilde{O}(\sqrt{n} \log{(1/\epsilon)})}$-time algorithm for agnostically learning conjunctions.

One can easily show that this algorithm is easily captured via learning a \emph{\regular} linear separator and thus fits into our framework (see \cref{sec:upper-bound}). However, our next result shows that somewhat disappointingly, kernel methods cannot yield an improvement over state of the art result. This is true even if we allow the learner to choose the kernel in a distribution dependent manner. We refer the reader to \cref{prf:lower} for a full proof.

\begin{restatable}{theorem}{lowb}\label{thm:lower}
There exists a distribution $D$ on $\X_n \subseteq \zo^n$ and a conjunction $c_I \in \conj$ such that for \emph{every} \regular RKHS $H$ and $\ww\in H$: for all $\ww$ such that $\|\ww\|_H =2^{\tilde{o}(\sqrt{n})}$, we have that
\begin{align*}
\EE{|\dotp{\ww}{\phi_{H}(\xx)} - c(\xx)|}> \frac{1}{6}.
\end{align*}
\end{restatable}

\ignore{
\Pnote{Not sure why we are stating this as a main result - why's this important/interesting?}
\Rnote{You mean corollary 5.1? yes we can move it to the conjunction setting, I guess: it's trivial.}
A different and well-studied approach for agnostic learning is through approximating every concept in the target class by a linear combination of a small number of functions which can then be learned using $\ell_1$-regression \cite{kalai2008agnostically}. As an immediate corollary of \thmref{thm:main} we obtain that if there's a \regular RKHS that allows efficient agnostic learning of a concept class, then one can also use $\ell_1$-regression to learn it.  
\begin{corollary}
Consider the $0-1$ loss function $\ellzo$.
Let $\H$ be an hypothesis class and assume that there exists \emph{some} $H\in \johnson$ and $\|\ww\|_{H} \le B$, such that
\begin{align*}
\EE{|\dotp{\ww}{\phi_{H}(\s{x}{i})} - h^*|}\le \frac{\epsilon}{2}.
\end{align*}
where $h^*=\arg\min_{h\in \H} \loss_{0,1}(h)$.
Then we can learn $\H$ efficiently within accuracy $\epsilon$. Specifically we can return a target function $f^*\in \C(\frac{C}{\epsilon})$, in time $\mathrm{poly}(B,1/\epsilon,\log 1/\delta)$ such that
\begin{align*}
\loss_{0,1}(f^*) \le \min_{h\in H}\loss_{0,1}+\epsilon
\end{align*}
\end{corollary}

\subsubsection*{Learning Conjunctions via Kernels}\roi{Rewrite this section seems trivial with the restatement of last result in terms of universal kernel}
We next consider applications of \regular kernels in learning the concept class of Boolean conjunctions: $\conj=\{c_{I}: c_I(\xx)=\wedge_{i\in I} \xx_i \}$ over the hypercube $\X_n$.\ignore{ Our next result shows that a sparseness assumption translates into a large margin assumption on some \regular kernel space:
\begin{theorem}\label{thm:upper}
Consider the $0-1$ loss function $\ellzo$.
Let $H$ be the hypothesis class of conjunctions and assume $D$ is a distribution supported on $s$-sparse vectors i.e. $D(\sum \xx_i \le s) =1$. Then for every $c_{I}\in \conj$ there exist $f_{H,\ww}\in \C(O(2^{s}))$ such that a.s.
\begin{align*}
\dotp{\ww}{\phi_{H}(\xx)} = c_I(\xx).
\end{align*}
As a corollary, we can learn the class in time $\mathrm{poly}(2^s,\frac{1}{\epsilon}, n)$.
\end{theorem}}
The best known result for learning conjunctions requires time $2^{\tilde{O}(\sqrt{n} \log{(1/\epsilon)})}$ (see. \cite{paturi1992degree,kalai2008agnostically}). These can be easily translated into a \regular kernel method framework (see \thmref{thm:cupper}). Our next result shows that in general, \regular kernel methods cannot yield a $2^{o(\sqrt{n})}$ time algorithm for the problem. 
\begin{theorem}\label{thm:lower}
There exists a distribution $D$ on $\X_n \subseteq \zo^n$ and a conjunction $c_I \in \conj$ such that for \emph{every} \regular RKHS $H$ and $\ww\in H$: for all $\ww$ such that $\|\ww\|_H =2^{o(\sqrt{n})}$, we have that
\begin{align*}
\EE{|\dotp{\ww}{\phi_{H}(\xx)} - c(\xx)|}> \frac{1}{6}.
\end{align*}
\end{theorem}
Observe that the above theorem shows that there's a fixed ``bad'' distribution and a ``bad'' conjunction for \emph{all} \regular RKHS embeddings.
It is instructive to compare our result with that of \cite{klivans2007lower} who showed that for every collection of $2^{o(\sqrt{n})}$ basis functions $\eta_1, \eta_2, \ldots, \eta_M$, there's a distribution $D$ on $\X_n \subseteq \zo^n$ and a conjunction $c$ such that $\inf_{\alpha_1, \alpha_2, \ldots, \alpha_M} \E_{x\sim D}[ |c(x) - \sum_{i \leq M} \alpha_i \eta_i(x)|] > \frac{1}{3}.$ Such a result rules out any set of fixed basis functions that can linearly approximate \emph{all} conjunctions and can be translated easily into showing for any fixed RKHS, there's a conjunction that will require a $2^{\Omega(\sqrt{n})}$-norm linear classifier. However, as we noted before, our algorithm from the previous section allows us to choose the best \regular kernel Hilbert space depending on the given conjunction and \emph{then} approximate it via a linear functional in this space and thus, a priori, this setting is not immediately captured by the above result. Nevertheless, we show that the lower bound carries over to our setting. The proof itself is not hard and uses a simple application of dimension reduction combined with a min-max result that holds because of the convexity of the problem of learning the optimal kernel.

To summarize, while \regular kernel methods capture the state-of-the-art algorithm for the problem, they do not help beat the $2^{\Omega(\sqrt{n})}$ barrier for agnostically learning conjunctions on arbitrary distributions.

While the above results are somewhat disappointing for an algorithm designer, we show that if the marginal distribution over examples is supported on sparse vectors, then \regular kernels do indeed yield an improvement over best previously known methods. Our conclusion is that under distributional assumptions, our algorithm for searching the optimal kernel may outperform existing methods.

}

\section{Technical Overview}
We next give a brief overview at a high level of our techniques:
\paragraph{Reduction to the hypercube layer}
We first observe that in order to show that the class is efficiently learnable over the hypercube, it is enough to restrict attention to the setting where the input distribution $\cD$ is supported on $S_{p,n}$ where $S_{p,n} = \{ \xx \in \zo^n \mid \sum \xx_i = p\}$ - the $p^{th}$ layer of the hypercube. 

Our reduction to the hypercube layer involves two steps. First we observe that the class $\C_n(B)$ is contained in $\C_{\oplus_n}(\sqrt{n}B)$. Namely we can replace every RKHS with an RKHS that can be presented as the Cartesian product over the different layers and lose at most factor $\sqrt{n}$ in term of margin. Thus, instead of learning \regular kernels, we restrict our attention to $\H_{\C_{\oplus_n}}$ which is expressive enough. This relaxation is exploited in both \cref{thm:main,thm:mkl}, hence both sample complexity result carry at least a linear factor dependence on the dimensionality in terms of sample complexity.

Working in $\H_{\C_{\oplus_n}}$ simplifies our objective. Since each RKHS in $\H_{\C_{\oplus_n}}$ is a direct sum of $n$ RKHS-s on each hypercube layer, we can focus on learning each component separately, and we derive efficient algorithms for learning $\C_{p,n}(B)$ for every $p=1,\ldots, n$. Thus, in \cref{thm:main} we construct a universal kernel over each hypercube layer. Meaning, we construct a Hilbert space $\universal_n^p$ such that $\C_{p,n}(B)$ is contained in $\universal^p_n((n+1) B)$. Finally, we sum up the universal kernels to construct a universal kernel over the Cartesian product of the layers. 

The approach suggested offers a simple method to learn $\C_n(B)$. The contruction of a universal kernel, though, causes a deterioration of an additional $O(n^2)$ factor in sample complexity. Our second approach in \cref{thm:mkl} suggests an efficient algorithm for learning the optimal RKHS in each hypercube layer $\C_{p,n}(B)$ directly and avoid a second relaxation.

Both the results, the existence of a universal kernel and the feasibility of learning the optimal RKHS rely on the special structure of kernels in $\H_{\C_{p,n}}$  which we next describe:
\ignore{
\cref{thm:solid} which is concrened with the domain of the solid cube is also achieved through a similar reduction. We show the domain can be discretized to reduce to a problem where the example points come from a hypercube of slightly larger dimension (see Section \ref{sec:kernel-learning-ball}).}
\paragraph{Characterizing \regular kernel through Johnson Scheme}
In this part we discuss what is arguablly the technical heart of our paper. Namely the application of classical results about the spectra of Johnson scheme matrices for the analysis of \regular kernels.

Consider any kernel over any layer $S_{p,n} = \{\xx \in \zo^n \mid \sum \xx_i = p\}$ of the $n$-hypercube - these are characterized by psd matrices indexed by elements of $S_{p,n}$ on the rows and columns. Searching over the class of all kernels thus involves searching over the space of all positive semidefinite matrices in ${n \choose p} \times {n \choose p}$ dimensions and is prohibitive in cost for $p = \omega(1).$

The main observation behind our algorithm is that while the assumption of \regular kernel allows us to capture almost all the kernels used in practice, it also allows for an efficient characterization of psd matrices defining them. In particular, recall that a \regular kernel matrix over $S_{p,n}$ is a matrix with any $(x,y)$-entry being a function solely of the inner product $\langle x, y \rangle.$ \ignore{In particular, such matrices satisfy a large amount of symmetry - such a matrix is invariant under the natural renaming action on the row and column vectors of any permutation from $S_n$. In other words, $K(x,y) = K(\sigma(x), \sigma(y))$ for every $\sigma \in S_n.$} Such matrices are called  \emph{set-symmetric matrices} and form a \emph{commutative algebra} called the \emph{Johnson association scheme}: the space of such matrices is closed under addition and matrix multiplication and any two matrices in the space commute w.r.t matrix multiplication. We provide more background on the Johnson scheme in Section \ref{sec:johnson}.

Standard linear algebra shows that a commutative algebra of matrices must share common eigenspaces. More interestingly, for our setting, the eigenspaces of set symmetric matrices have been completely figured out in the study of Johnson scheme. In particular, despite the matrices themselves being of dimension ${n \choose p} \times { n \choose p}$, they can have at most $p+1$ distinct eigenvalues! Further, there's a positive semidefinite basis of $p+1$ matrices $\{P_{p,\ell} \mid \ell \leq p\}$ for the linear space of Johnson scheme matrices with tractable expressions for eigenvalues in the $p+1$ different eigenspaces.

\paragraph{Construction of a universal Hilbert Space}

Equipped with an explicit basis of set symmetric matrices, and having a diagnolized representation for the matrices, we can explicitly construct $p+1$ kernel matrices $K_1,\ldots, K_{p+1}$ whose convex hull spans all set symmetric, positive definite and bounded by $1$ matrices. These are the matrices that correspond to a \regular kernel.
Thus we obtain an explicit characterization of the polytope of \regular kernels in terms of $p+1$ vertices where each kernel is a convex combination of the vertices.

Using the above construction we finally consider the direct sum Hilbert space $\universal^p_n= H_1\oplus H_2,\ldots,\oplus H_{p+1}$, where the $H$'s correspond to the kernel vertices. A direct corollary of the above characterization is that any target function in $\C_{n,p}(B)$ may be written in the form of $f_{H',\ww}(x) = \sum \lambda_i \ww_i\cdot \phi_i(x)$ where $\ww_i\in H_i$. Standard linear algebra then show we can bound the norm of the above target function in terms of $\|\cdot\|_{\universal}$ by losing a factor of at most $\sqrt{n}$. Thus $\C_{n,p}(B)$ is a subset of all $\sqrt{n}B$ bounded norm vectors in $\universal^p_n$.

\paragraph{Improving sample complexity through MKL}
The above results allow us to efficiently learn the class $\C_n(B)$ however it may lead to suboptimal result in sample complexity. As we next discuss this can be improved by optimizing over the choice of kernel as is done in MKL.

First, as discussed before, \emph{any} matrix of the Johnson scheme can be specified by describing the $p+1$ coefficients over the basis - and one can write down explicit expressions in these coefficients for the eigenvalues. Thus, checking PSDness reduces to just verifying $p+1$ different linear inequalities.

The above observation allows us to take the standard $\ell_2$-regularized kernel SVM convex formulation and add an additional minimization over the space of coefficients that describe a \regular kernel. We show that the resulting modified program is convex in all its variables. Similar observations on the convexity of such programs have been made in previous works starting with the work of \cite{lanckriet2004learning}. Together with the tractable representation of the constraint system we obtained above, we get an efficiently solvable convex program\full{ (see \thmref{lem:primal})}.

To achieve generalization bound, we appeal to the surprisingly strong bounds on Rademacher complexity of non-negative linear combinations of $q$ \emph{base kernels} due to \cite{cortes2010generalization}. Our generalization bounds follow a certain strenghening of the aforementioned result to $\C_{\oplus_n}(B)$. In turn, using the fact that the polytope of \regular kernels has exactly $p+1$ vertices, we can derive strong sample complexity upper bound that grows only logarithmically in the dimension $n$.

\paragraph{Limitations for Learning Conjunctions}
Our final application for learning kernels helps in proving bounds on the expressive power of the family of large margin linear classifiers in \regular RKHS.
Our crucial observation relies on a result by \cite{klivans2007lower} who showed that for every collection of $2^{o(\sqrt{n})}$ basis functions $\eta_1, \eta_2, \ldots, \eta_M$, there's a distribution $D$ on $\X_n \subseteq \zo^n$ and a conjunction $c$ such that $\inf_{\alpha_1, \alpha_2, \ldots, \alpha_M} \E_{x\sim D}[ |c(x) - \sum_{i \leq M} \alpha_i \eta_i(x)|] > \frac{1}{3}.$ Such a result rules out any set of fixed basis functions that can linearly approximate \emph{all} conjunctions.

Our first step in the proof translates the aforementioned result to showing that for any fixed RKHS, there's a conjunction that will require a $2^{\Omega(\sqrt{n})}$-norm linear classifier. We do that by showing that any fixed kernel with a separator with large margin will yield, via Johnson--Lindenstrauss , a small class of basis functions that can approximate any conjunction. However, this technique alone does not capture the possibility of learning the kernel Hilbert space  and \emph{then} approximate it via a linear functional in this space. In other words, while the aforementioned result restrict the expressive power of each specific kernel, it does not put limitations over $\C_n(B)$.

However, the existence of a universal Hilbert space demonstrates that the power of \regular kernels cannot exceed any limitation over a fixed kernel. Thus, building upon \cite{klivans2007lower} we obtain a uniform lower bound for the expressive power of $\C_n(B)$ and in particular $\C_n(\bbB_n,B)$:
\ignore{
\paragraph{Learning Conjunctions}
Our next result is about a lower bounds on \regular kernel  methods for learning conjunctions.

Klivans and Sherstov \cite{klivans2007lower} showed via lower bounds on approximate rank of certain matrices that given a collection of function $\zeta_1, \zeta_2, \ldots, \zeta_q$ for any $q = 2^{o(\sqrt{n})}$, there's a distribution over the hypercube and a conjunction $c$ on $n$ bits such that $c$ cannot be approximated by a linear combination of $\zeta_1, \zeta_2, \ldots, \zeta_q$. We discussed in the introduction why these lower bounds do not immediately apply to our setting of learning an optimal \regular kernel for a given set of samples and then learning the optimal linear classifier in the associted RKHS.  


Our proof is simple and builds on the work of \cite{klivans2007lower}: First we show that there is no fixed \regular kernel that can approximate every conjunction on every layer uniformly. This follows from a standard dimension reduction argument together with \cite{klivans2007lower}: we project any large margin approximator into a low dimension subspace and use the lower bound on the dimension of the subspace required to approximation conjunctions from \cite{klivans2007lower} yielding a lower bound on the margin to begin with. To reverse the order quantifiers and to show that for a fixed distribution and a fixed conjunction, no \regular kernel works, we apply a min-max principle. To use minmax principle, we need to argue that the the kernel-learning program is convex - this is indeed true and simple to verify.}


\section{Background}

\subsection{Johnson Scheme}\label{sec:johnson}
In this section, we describe the \emph{Johnson Scheme} (or set-symmetric) matrices that are an instance of \emph{association schemes}, a fundamental notion in algebraic combinatorics and coding theory. We will need the classical result about the eigendecompositions of such matrices in this work. We refer the reader to the textbook and lecture notes by Godsil for further background \cite{MR1220704-Godsil93}.

\begin{definition}[Johnson Scheme]
Fix positive integers $t,n$ for $t < n/2$. The Johnson scheme with parameters $t,n$, denoted by $\cJ_{n,t}$ is a collection of matrices with rows and columns indexed by subsets of $[n]$ of size exactly $t$ such that for any $M \in \cJ_{n,t}$ and any $S, T  \subseteq [n]$ of size $t$, $M(S,T)$ depends only on $|S \cap T|$. 
\end{definition}
That is, any entry of a matrix $M\in \cJ_{n,t}$ depends only on the size of the intersection of the subsets indexing the corresponding row and column. Equivalently, we can think of the matrices in the Johnson scheme as indexed by elements of $\zo^n$ of Hamming weight exactly $t$ with $(x,y)$th entry a function of the inner product $\langle x, y \rangle$. The symmetric group on $n$ elements $\bbS_n$ acts on subsets of size $t$ of $[n]$ by the natural renaming action and further, $|S \cap T| = |\sigma(S) \cap \sigma(T)|$ for any  permutation $\sigma \in \bbS_n$. Thus, $M$ is invariant under the action of $\bbS_n$ that renames its rows and columns as above. 

It is not hard to verify that $\cJ_{n,t}$ forms a commutative algebra of matrices. A basic fact in linear algebra then says that the matrices in $\cJ_{n,t}$ must share common eigen-decomposition. The natural action of $\bbS_n$ associated above makes the task of pinning down a useful description of this eigenspaces tractable - these form classical results in algebraic combinatorics. This description of eigenspaces of the Johnson scheme will come in handy for us and in the following, we will describe the known results in a form applicable to us. 

It is convenient to develop two different bases for writing the matrices in $\cJ_{n,t}$.  
\begin{definition}[$D$ Basis]
For $0 \leq \ell \leq t <n$, we define the matrix $D_{n,t, \ell} \in \R^{{[n] \choose t }\times {[n] \choose t}}$ by: $D_{n,t,\ell}(S,T) = 1$ if $|S \cap T| = \ell$ and $0$ otherwise.
\end{definition}
It is easy to see that every matrix in $\cJ_{n,t}$ can be written as a linear combination of the $D_{n,t,\ell}$ matrices for $0 \leq \ell \leq t.$ Further, it's easy to check that any pair of $D$ matrices commute with each other and thus so does every pair of matrices from the Johnson scheme. 

While the $D$ basis is convenient to express any matrix in the Johnson scheme, it's not particularly convenient to uncover the spectrum of the matrices. For this, we adopt a different basis, called as the $P$ basis. 

\begin{definition}[$P$ Basis]\label{def:pbasis}
For $0 \leq t \leq t$, let $P_{t,p} \in \R^{{[n] \choose t }\times {[n] \choose t}}$ be the matrix defined by: $P_{t,p}(S,T) = {|S \cap T| \choose \ell}$ where we think of ${r \choose \ell}$ for $r < \ell$ as $0$. It is easy to check that $P_{t,p}$ is positive semidefinite for all $t$ and linearly spans $\cJ_{n,t}.$
\end{definition}

The following translation between the $P$ and the $D$ basis is easy to verify.
\begin{fact}[Basis Change]
Fix $r\leq t < n$. Then, 
\begin{enumerate}
\item $P_{p,r} = \sum_{\ell = r}^t {\ell \choose r} D_{\ell}.$ 
\item For $0 \leq \ell \leq t$, $D_{\ell} = \sum_{r \geq \ell} (-1)^{r-\ell}{r \choose \ell}P_{p,r}.$ 
\end{enumerate}
\end{fact}

The $P$ basis helps us write down a simple expression to compute the eigenvalue of any matrix in the Johnson scheme, given that we know how to write it as a linear combination of the $P_{p,t}$ matrices. The following result is what makes this possible.  

\begin{fact}[Eigendecomposition of the Johnson Scheme, Eigenvalues of $P_{p,t}$] \label{fact:eigendecomp-johnson}
Fix $n, t < n/2$. There are subspaces $V_0, V_1, \ldots, V_{t}$ such that $\R^{[n] \choose t} = \oplus_{i\leq t}V_i$ satisfying:
\begin{enumerate}
\item $V_0, V_1, \ldots, V_t$ are the eigenspaces of every matrix in the Johnson scheme $\cJ_{n,t}$. 
\item For $0 \leq j \leq t$, $V_j$ is of dimension ${n \choose j} - {n \choose {j-1}}$ (where we define ${n \choose -1}  = 0$.) 
\item Let $\lambda_j(Q)$ for $0 \leq j \leq t$ denote the eigenvalue of $Q \in \cJ_{n,t}$ on the eigenspace $V_j$. Then, 
\[
\lambda_j(P_{p,\ell}) = \begin{cases} {{n-\ell-j} \choose {t-\ell}} \cdot {{t-j} \choose {\ell-j}} & \text{ if } j \leq \ell\\
0 & \text{ otherwise.}
\end{cases}
\]
\end{enumerate}
\end{fact}

\ignore{
\paragraph{Basic Notation} $\X$ will denote the domain of the concept class and $\Y$ the range throughout the paper - in all settings we encounter, $\X \subseteq \R^n$ and one can without loss of generality think of $\Y \subseteq [-1,1] \subseteq \R.$ In this paper, we will deal with the standard definition of loss functions encountered in machine learning. In this setting, a \emph{loss function} $\ell: \R \times \Y \rightarrow \R^{+}$ is a function that is 1) convex in the first argument: $\ell(x,y)$ is convex in $x$ for every $y$ 2) $1$-Lipschitz in its first argument i.e., $|\ell(x_1,y)-\ell(x_2,y)| \leq \|x_1-x_2\|_2$ for every $x_1,x_2$ 3) bounded at zero: namely, $\ell(0,y)\leq 1$ for every $y \in \Y.$ Whenever we say loss function in this paper, we mean a function that satisfies all the three properties listed here with respect to the $\Y$ clear from the context.

Convexity is a standard assumption for tractability of ERM in Machine learning. The second assumption is required for generalization bounded using Rademacher complexity and the third one is a normalization that holds without loss of generality. In the second assumption, we could more generally deal with $B$-Lipschitz loss functions but we do not in order to avoid unnecessary notational clutter. \Rnote{What??}}

\ignore{
\paragraph{Agnostic Learning (\cite{DBLP:journals/iandc/Haussler92,DBLP:journals/ml/KearnsSS94})} We consider the standard setting of agnostic learning a class $\cF$ of functions over $\X \rightarrow \Y$  with respect to an arbitrary $1$-Lipschitz and loss function $\ell$ satisfying: $\X \subseteq \R^n$, $\Y \subseteq \R$, $\ell(0,y) \leq 1$ for every $y \in \Y$. An agnostic learner for $\cF$ with respect to a distribution $\cD$ over $\X \times \R$ takes input i.i.d. samples from $\cD$ and outputs a hypothesis $h$ (not necessarily from $\cF$) such that $\E_{(x,y) \sim \cD}[ \ell(h(x),y)] \leq opt(\cF) + \epsilon$ with probability at least $2/3$ (this can be boosted to any $1-\delta$ with a multiplicative loss of $\poly(\log{(1/\delta}))$ in running time and sample complexity, which we omit in order to simplify notation). Here $\epsilon$ is the accuracy parameter and $opt(\cF) = \inf_{f \in \cF} \E_{(x,y) \sim \cD}[\ell(f(x), y)].$ We measure the performance of the algorithm by bounding its sample complexity - the number of samples required and the running time, as a function of the dimension $n$ and the parameter $\epsilon$. We say that the algorithm is efficient if the running time and sample complexity is polynomial in $n$ and $\frac{1}{\epsilon}$.
}


\ignore{
\subsection{Kernel Methods}
Here we overview the basic ideas and main observations regarding kernels that we will use. we refer the reader to \cite{scholkopf2001learning} or any other standard text book.

Consider a Hilbert space $H$ and a norm bounded embedding $\phi:\X\to H$. We let $k$ denote the kernel function $k:\X\times \X \to \mathbb{R}$ defined by $k(x_i,x_j) = \dotp{\phi(x_i)}{\phi(x_j)}$. A Hilbert space $H$ together with an embedding $\phi$ is be referred to is an RKHS. Our key assumption, when applying the kernel method, is that $k$ is an efficiently computable function.

Given an RKHS $H$ and kernel function $k$, We may consider the solution to the following (\regularized) empirical risk minimization objective:
\begin{align}\label{eq:kerm}
\mathop\mathrm{minimize}_{\ww \in H}~\quad \frac{\lambda}{2} \|\ww\|^2 +\EE{\ell(\dotp{\ww}{\phi(\xx)},y}
\end{align}
One way to tackle the problem is to apply a kernelized SGD algorithm \cite{shalev2007pegasos}, which is guaranteed to find an $\epsilon$ approximate solution using $O(\frac{\rho}{\lambda \epsilon})$ iterations. By choosing $\lambda=\frac{B}{\epsilon}$, and noting that any solution $\|\ww\|> B$ is suboptimal We can achieve the following corollary:
 \begin{theorem}[\cite{shalev2007pegasos} Cor. 2]
Applying the kernelized version of Pegasos algorithm to solve \eqref{eq:kerm}, we can find a classifier of the form $\ww= \sum_{i=1}^m \alpha_i \phi(\s{x}{i})$ such that if $m=\tilde{O}( \frac{B}{\epsilon^2}\log 1/\delta)$ then with probability at least $1-\delta$ we have that
\begin{align*}
\loss(\ww ) \le \min_{\|\ww^*\|<B} \loss(\ww^*) + \epsilon.
\end{align*}
As a corollary, if $k(x_i,x_j)$ is efficiently computable, we can efficiently learn the class $\{\xx\to \dotp{\ww}{\phi(\xx)} : \|\ww\|<B\}$.
\end{theorem}}

\subsection{Kernel Method: Learning Linear Classifiers in RKHS}\label{sec:prerkhs} We now recall the standard framework for agnostically learning linear classifiers in a RKHS (see \cite{scholkopf2001learning} for a detailed overview). 
\begin{definition}[RKHS for $\X$ and Kernels]
Let $H$ be a Hilbert space with an inner product $\langle \cdot, \cdot \rangle_{H}$ and the corresponding norm $\|\cdot \|_H$) along with an embedding $\phi:\X \rightarrow H$. $H$, together with the embedding $\phi$ is said to be an RKHS for $\X$.
\end{definition}
For any RKHS $(H, \phi)$, there's a unique \emph{kernel function} $k: \X \times \X \rightarrow \R$ that is defined by the inner products of any two elements of $\X$ embedded in $H$: i.e., $k(\s{x}{1},\s{x}{2}) = \langle \phi(\s{x}{1}), \phi(\s{x}{2}) \rangle_H$.  When $\X$ is finite, $k$ is completely described by the $|\X| \times |\X|$ \emph{kernel matrix} whose rows and columns are  indexed by elements of $\X$ and any $(\s{x}{1}, \s{x}{2})$ entry being given by $k(\s{x}{1} , \s{x}{2})$. A classical result in kernel theory, namely Mercer condition, states that a function $k$ is the kernel function on an RKHS if and only if the corresponding kernel matrix $K$ is psd. In applications, we'd also want the function $k$ to be efficiently computable (w.r.t the natural parameters of the problem).

Consider the class of all functions of the form $f_{\ww}: \X \rightarrow \R$ defined by $f_{\ww}(x) =\dotp{\ww}{\phi(x)}$ (where we suppress the subscript $H$ when there is no room for confusion). These are linear functions in the Hilbert space extended to $\X$ via the embedding $\phi$. The key observation underlying kernel methods is that the class of all such linear functions where the coefficient vector $\ww$ satisfies $\|\ww\|_H < B$ is efficiently learnable.  Via standard primal-dual analysis (encapsulated by the "representer theorem"), one can show that the solution to the above convex program can be written as $h(\xx) = \sum_{i \leq m} \alpha_i k(\s{x}{i}, \xx)$ for the kernel function $k$ associated with $H$. The following theorem captures the error and generalization bounds one can show for solving the above convex minimization program. The sample complexity analysis is based on the SGD based method to approximately solve the convex program above presented in \cite{shalev2007pegasos}.

\begin{fact}[See \cite{shalev2007pegasos} for instance, for a proof] \label{fact:pegasos-sgd-svm-guarantee}
Let $\X,\Y, \phi, H, k$ be as defined above. There exists an algorithm that takes as input an i.i.d sample $S$ of size $m = m(n,\epsilon,\delta)$ and with probability at least $2/3$ over the sample, outputs a hypothesis $h:\X \rightarrow \R$ defined as $h(\xx) = \sum_{i \leq t} \alpha_i k(\s{x}{i},\xx),$ for scalars $\alpha_i$ satisfying $\sum_{i \leq t} |\alpha_i| \leq \frac{B^2}{\epsilon} $ that satisfies: $\loss_{D}(h) \leq opt_{D}(H(B)) + \epsilon$. The running time and the sample complexity of the algorithm is $O(\frac{B^2}{\epsilon^2})$.
\end{fact}

\ignore{
\subsection{Rademacher Complexity and Generalization bounds}

While learning a kernelized linear classifier is known to be easy, here we consider the much more challenging task of finding linear classifier and a correct kernel space. Our strategy is to choose a Hilbert space and classifier that minimize the empirical error. A remarkable strong generalization bound for this approach was developed in \cite{cortes2010generalization} who showed that, when the class to be learnt is a convex hull of say $n$ kernels, the sample complexity scales logarithmically with $n$.   We will need a slightly stronger result for a slightly more complex structure, but our generalization bound is essentially due to \cite{cortes2010generalization} through a slight modification.  The work of Cortes et al. relies on bounding the Rademacher complexity of the class, and we now briefly recall this term and the relevant generalization bound. 

Recall that the Rademacher Complexity of a class $\H$ over a sample $S=\{\s{x}{1},\ldots, \s{x}{m}\}$ is defined as follows
\[R_m(\H,S) = \mathbb{E}_{\sigma} \left[ \sup_{f\in \H}   \frac{1}{m}\sum_{i=1}^m \sigma_{i} f(\s{x}{i})\right]\]
where $\sigma \in \{-1,1\}^t$ are i.i.d.~Rademacher distributed random variables.
The following bound the generalization performance of an empirical risk minimizer with respect to the class $\H$. (e.g. \cite{shalev2014understanding, bartlett2002rademacher})

\begin{fact} \label{fact:rad}
Let $\ell$ be a $1$-Lipschitz convex loss function with $|\ell(0,y)|\le 1$. Assume that for all $\xx$ and $f\in \H$ we have $|f(\xx)|<c$.  Given an IID sample from $\cD$ supported over $\X \times \Y$, for any $f\in \H$ with probability at least $1-\delta$ (over $S$):

\begin{align}\label{eq:radbound}
\loss_{S_m}(f) \le  \loss_{\cD}(f) + 4 \sup_{S_m} R_m(\H,S)+ 4c\sqrt{\frac{2\ln 2/\delta}{m}}
\end{align}
\end{fact}
\ignore{
Rademacher complexity has been used in \cite{cortes2010generalization} to obtain generalization bounds for learning a convex hull of $p$ kernels: Fix a set of $p$ kernels $k^1,\ldots,k^p$, such that $k(\xx,\xx)\le 1$. For each $\gamma\ge 0, ~\sum \gamma_i=1$ we consider $k_{\gamma} = \sum \gamma_i k^{i}$, then $k_{\gamma}$ is a kernel matrix that induces an embedding $\phi_{\gamma} :\X \to H_{\gamma}$ to a Hilbert space. We may consider the following class of target functions
\begin{align}\label{eq:Hb}
\H_{\{k_1, k_2, \ldots, k_p\}}(B) = \{f_{\ww,H_{\gamma}} : \xx \to \dotp{\ww}{\phi_{\gamma}(\xx)} \mid \gamma:~ \|\ww\|_{H_{\gamma}}\le B\}
\end{align}

The following result, bounds the Rademacher complexity of the class $\H_{B}$:
\begin{fact}\cite{cortes2010generalization}\label{fact:cortes}
Let $k^1,\ldots, k^p$ be kernels such that $k^i(\xx,\xx)\le 1$ and set $\H_{\{k_1, k_2, \ldots, k_p\}}(B)$ as in \eqref{eq:Hb}, then for any sample $S_m$ 
\begin{align*}
R_m(\H_{\{k_1, k_2, \ldots, k_p\}}(B),S) = O\left(\sqrt{B^2 \frac{ \log p}{m}}\right)
\end{align*}
\end{fact}
} 
}

\bibliographystyle{plain}
\bibliography{bib/klearning,bib/mathreview,bib/dblp,bib/scholar,bib/custom}

\appendix
\renewcommand{\jvec}{\eta}
\section{Proof of \cref{thm:main}}\label{prf:main}
This section is devoted to prove \cref{thm:main} which we now restate.
\main*
The proof involves two stages. First we show a reduction from the hypercube case to the hypercube layer. Namely, we show that if we can construct a universal kernel for each layer, then we can also construct a universal kernel for the hypercube. Then we proceed to construct a universal kernel for each layer. Finally, we give a full proof at the final section \cref{sec:mainproof}.

\subsection{Reduction to the hypercube layer $S_{p,n}$}
Our first step will be to reduce the problem of constructing a universal kernel over the hypercube, to the construction of universal kernels over the hypercube layers. For this we first recall the subclass of \regular kernels of all kernels that can be decomposed to a Cartesian product over the layers -- $\C_{\oplus_n}$. We next show that $\C(B)\subseteq \C_{\oplus_n}(\sqrt{n}B)$ as a corollary, constructing a universal Hilbert space for $H_{\C_{\oplus_n}}$ is sufficient. We then show that if we can construct a universal Hilbert space on each layer, by taking their Cartesian sum, we can construct a universal Hilbert space over $\C_{\oplus_n}$.

\begin{lem}\label{lem:impropermkl}
For every $n$ we have the following inclusion:
\[\C(B)\subseteq \C_{\oplus_n}(\sqrt{n} B).\]
\end{lem}
\begin{proof}
Let $H$ be a \regular RKHS and let $H_1,\ldots, H_n$ be the projections of $H$ onto $\left(\mathrm{span}(\s{x}{i})_{i\in S_{1,n}},\ldots,\mathrm{span}(\s{x}{i})_{i\in S_{n,n}}\right)$ respectively. We then take the space $\bar H=H_1\oplus H_2,\ldots, \oplus H_n$ and the embedding 
$\bar\phi(\xx)=(0,0,\ldots,\underbrace{\phi(\xx)}_{\sum \xx_i =p},0,\ldots,0)$, with associated kernel $\bar k(\s{x}{i},\s{x}{j})=\begin{cases} k(\s{x}{i},\s{x}{j}) & \|\s{x}{i}\|=\|\s{x}{j}\| \\
0 & \mathrm{else}
 \end{cases}$. Then one can show that $f_{H,\ww} = f_{\bar H,(\ww_1,\ldots,\ww_n)}$ where $\ww_p$ is the projection of $\ww$ onto $H_p$. Overall we have that
 \begin{align*}\|(\ww_1,\ldots,\ww_n)\|_{\bar H} = \sqrt{\sum \|\ww_i\|^2_{H_i}}=\sqrt{\sum \|\ww_i\|^2_{H}}\le \sqrt{\sum \|\ww\|^2_{H}}\le \sqrt{n} B\end{align*}
 \end{proof}

\begin{lem}[Learning \regular Linear Separators over the Hypercube] \label{lem:mainreduction}
Fix $n$ and let $k_1,\ldots,k_n$ be kernels associated with universal RKHS $((\universal_n^1,\phi_1),\ldots, (\universal_n^n,\phi_n))$ such that for all $B>0$ and $p=1,\ldots, n$:
\[\C_{p,n}(B)\subseteq \universal_n^p(\alpha B).\]

Let $k$ be the \regular kernel associated with the Hilbert space $\universal = \universal_{n}^1\oplus ,\ldots,\oplus \universal_n^n$ together with the embedding 
\[\phi(x)=(0,0,\ldots,\underbrace{\phi_t(\xx)}_{t\textrm{-th coordinate}},\ldots,0), \quad \forall \sum \xx_i=t.\]
Then:
\[ \C_{\oplus_n} (B)\subseteq \universal(\alpha B),\]
and Computing $k$ takes $O(n + T(n))$ where $T(n)$ is the time complexity for the kernels $k_1,\ldots,k_p$.
\end{lem}
\begin{proof}
Choose $f_{H,\ww}\in \C_{\oplus_n}(B)$ for some $\ww=(\ww_1,\ww_2,\ldots, \ww_n)\in H_1\oplus H_2,\ldots, \oplus H_n$. It is easy to see that for every $\xx\in S_{p,n}$ we have that $f_{H,\ww}(\xx) = f_{H_p,\ww_p}(\xx)$. Next, for each $\ww_p$ there exists $\vv_p\in \universal_n^p(\alpha \|\ww\|_p)$ such that $f_{H_p,\ww_p} = f_{\universal_n^p,\vv_p}$. Overall we get that $f_{H,\ww}=f_{\universal_n,\vv}$ where $\vv=(\vv_1,\ldots, \vv_n)$. It remains to bound the norm of $\vv$:
\[\|\vv\|_{\universal} = \sqrt{\sum \|\vv_p\|_{\universal_n^p}^2} \le \sqrt{ \sum \alpha \|\ww_i\|_{H_p}^2}\le \alpha \|\ww\|_H\le \alpha B\]
\end{proof}

\subsection{Learning over $S_{p,n}$}\label{sec:hyperlayer}

The main result of this section shows that there exists a universal Hilbert space over a single layer of the hypercube.

\begin{lem}[Learning \regular Linear Separators over a Single Layer]\label{lem:sp}

For fixed $n$ and every $B\ge 0$ the class $\C_{p,n}(B)$ is efficiently learnable w.r.t. any convex $L$-Lipschitz loss function $\ell$ in with sample complexity $O(p^2 B^2/\epsilon^2).$

Specifically, for every $p$ there exists an efficiently computable kernel associated with a universal \regular RKHS $\universal^p_n$, such that \[\C_{p,n}(B) \subseteq \universal^p_n\left((p+1)B\right).\]
The computation of $k$ involves a preprocessing stage of $O(p^3)$, and then the computation of each entry $k(\s{x}{i},\s{x}{j})$ is done in time $O(p)$.
\end{lem}

To prove \cref{lem:sp} we begin with a direct application of classical results on eigenspaces of the matrices of the Johnson scheme to obtain a useful characterization of \regular kernels over a single layer of the Boolean hypercube.

Let $\jvec^p\in \R^{p+1}$ be defined by $\jvec^p_{\ell} = {p \choose \ell-1}$ for every $\ell$ and define $\Delta^p \in \R^{p+1 \times p+1}$ by \begin{align}\label{eq:Delta}
\Delta^p_{j,\ell} = \begin{cases} {{n-\ell-j} \choose {p-\ell}} \cdot {{p-j} \choose {\ell-j}} & 0 \leq j \le \ell \\ 0 & \text{ otherwise.}
\end{cases}
\end{align}
For fixed $p$, corresponding to the $P$-basis of positive definite matrices in Definition \ref{def:pbasis} we will denote by $\kp_{t}$, the kernel over $S_{p,n}$ that is given by
\[ \kp_{t}(\s{x}{i},\s{x}{j}) = {{\s{x}{i}\cdot\s{x}{j}}\choose{t-1}}.\] Following the discussion in \secref{sec:johnson} and noting that the kernel matrix $K_{p,t}\in \mathbb{R}^{{p\choose t}\times {p\choose t}}$ equals a non-negative scaling of $P_t$ and is thus PSD, $\kp_{t}$ is a kernel over $S_{p,n}$.

\begin{lem}[Characterizing \regular Kernels] \label{lem:characterization-regular-kernel}Fix $p\le n/2$, and let $\{\s{x}{1},\ldots \s{x}{{n\choose p}}\}=S_{p,n}$. For a \regular kernel function over $S_{p,n}$ there exists an RKHS $(H,\phi)$ with associated kernel function $k$ if and only if there is a vector $\beta\in \R^{p+1}$ such that $k(\s{x}{i},\s{x}{j}) = \sum_{t \leq p+1}\beta_t \cdot \kp_{t}(\s{x}{i},\s{x}{j})$ satisfying:
\begin{enumerate}
\item $\Delta^p \beta \geq 0$ for $j=0,\ldots, p$. 
\item $\dotp{\eta^p}{\beta} \leq 1$
\end{enumerate}
\end{lem}

\begin{proof}
This is a direct application of Fact \ref{fact:eigendecomp-johnson}. Let $K \in \R^{S_{p,n} \times S_{p,n}}$ is defined by $K_{i,j} = k(\dotp{\s{x}{i}}{\s{x}{j}})$ for some kernel function $k$.  Observe that $K$ is a kernel matrix and corresponds to an RKHS $(H,\phi)$ if and only if $K$ is positive semidefinite, further we have that $\|\phi(\s{x}{i})\|\le 1$ if and only if $K(i,i) \leq 1.$ 

Since $K$ is a kernel matrix of a \regular kernel, in particular, it is set-symmetric (an entry only depends on the inner products of the row and column index vectors) and thus, shares eigenspaces with all the matrices in the Johnson scheme and in particular with the matrices $P_{p,\ell}$ that span the space. The $\beta_i$ are thus the coefficients in the $P$-basis for $K$ and allow us to write down the eigenvalues of $K$ as linear functions in $\beta$ and the fixed constant eigenvalues of $P_{p,\ell}$. By Fact \ref{fact:eigendecomp-johnson} and \cref{eq:Delta}, the first condition is then just the statements that all eigenvalues of $K$ be non-negative. The second condition checks that $K(x,x)$ is bounded by one
\end{proof}

The simple lemma above is surprisingly powerful. Even though the matrix $K$ is huge (of dimensions $n^p \times n^p$ roughly), verifying that it's PSD is easy and corresponds to just checking $p+1$ different linear inequalities in $p+1$ variables. A simple corollary of \cref{lem:characterization-regular-kernel} is that we can by change of variable, describe the set of \regular kernels as a polytope with $p$ vertices corresponding to kernels
\begin{corollary}\label{cor:vertices}
For each $i$ set $\s{\beta}{i}\in \mathbb{R}^{p+1}$ such that 
\begin{align}\label{eq:lineq}
\left(\Delta^p\right)\s{\bar\beta}{i} = e_i, \quad \s{\beta}{i}= \frac{\s{\bar\beta}{i}}{\dotp{\eta^p}{\s{\bar\beta}{i}}}
\end{align}
Then the kernel function $k_{p,i} = \sum \s{\beta}{i}_t \kp_{p,t}$ is indeed a kernel. Moreover every \regular kernel associated to an RKHS can be written as $k= \sum \lambda_i k_{p,i}$ where $\lambda_i\ge 0$ and $\sum \lambda_i\le 1$.

\end{corollary}
\begin{proof}
Let $\mathcal{K} \in \R^{p+1}$ be the set of all vectors $\beta$ such that $\sum \beta_t \kp_{t}$ is a \regular kernel associated with an RKHS. By \cref{lem:characterization-regular-kernel}, this set is convex and also $\s{\beta}{i}\in \mathcal{K}$.

 We next wish to show that $\s{\beta}{1},\ldots,\s{\beta}{p+1}$ contain all the vertices of the set $\mathcal{K}$. Indeed, recall that invertible affine transformations preserve the set of vertices. Set $\xi^{(p)} =\left(\Delta^{p}\right)^{-\top}\eta^p$ and consider the following set:
 \begin{align*}
 {\Delta^p}\mathcal{K}&= \{ \Delta^p \beta: \Delta^p \beta \ge 0,~ \dotp{\eta^p}{\beta}\le 1\}\\
 &= \{ v: v \ge 0,~ \dotp{\eta^p}{\left(\Delta^{p}\right)^{-1}v}\le 1\} \\
  &= \{ v: v \ge 0,~ \dotp{\left(\Delta^{p}\right)^{-\top}\eta^p}{v}\le 1\} \\
    &= \{ v: v \ge 0,~ \dotp{\xi^{(p)}}{v}\le 1\}.
 \end{align*}
 One can then observe that the set of vertices of the set ${\Delta^p}\mathcal{K}$ is given by $\{\frac{1}{{\xi^{(p)}}_i} e_i\}_{i=1}^{p+1}$. Finally observe that
\begin{align*}
\xi^{(p)}_i=\left((\Delta^p)^{-\top}\eta^p\right)_i =\dotp{(\Delta^p)^{-\top}\eta^p}{e_i} =\dotp{\eta^p}{(\Delta^p)^{-1} e_i}=\dotp{\eta^p}{\s{\bar\beta}{i}}  
\end{align*}
Taking the reverse image we obtain that the set of vertices of the set $\mathcal{K}$ are given indeed by $\s{\beta}{1},\ldots,\s{\beta}{p+1}$. By definition of $\mathcal{K}$ we obtain the desired result.
\end{proof} 

Finally we are ready to prove \cref{lem:sp}.
\paragraph{Proof of \cref{lem:sp}}
First, without loss of generality we may assume $p \le \frac{n}{2}$. If $p> \frac{n}{2}$ then we simply map $S_{p,n}$ into $S_{n-p,n}$ by having $\xx_i \to (1-\xx_i)$. Note that a \regular kernel remains a \regular kernel under this mapping.

Set $k_{p,1},\ldots, k_{p,p+1}$ be as in Corollary \ref{cor:vertices}, and define for each kernel its associated RKHS $(H^p_{1},\phi^p_1),\ldots, (H^p_{p+1},\phi^p_{p+1})$. 

We next define our candidate for a universal Hilbert space and consider the Hilbert Space \[\universal_n^p = H^p_1\oplus H^p_2 ,\cdots,\oplus H^p_{p+1}.\] Then it is not hard to see that $\universal_n^p$ forms an RKHS with the natural embedding and kernel 

\begin{align*}\phi^{u}(\xx)&:=\frac{1}{p+1}(\phi^p_1(\xx),\ldots, \phi^p_{p+1}(\xx))\\
k_p &:= \frac{1}{p+1}\sum_{i=1}^{p+1} k_{p,i}
.\end{align*}

Fix $f_{H,\ww}\in \C_{p,n}(B)$, the we need to show that $f_{H,\ww}\in \universal_n^p((p+1)B)$.

Denote by $H_{S}$ the projection of $H$ onto $\mathrm{span}\{\phi(\s{x}{i})\}_{\{\s{x}{i}\in S_{p,n}\}}$, where $\phi$ is the embedding onto $H$. Without loss of generality we can assume that $\ww\in H_S$. , Indeed, let $\ww' \in H_S$ be the projection of $\ww$ on $H_S$ then $\|\ww'\|<\|\ww\|\le B$ and we have that for all $\xx\in S_{p,n}$: $f_{H,\ww}(x) = \dotp{\ww}{\phi(x)}=\dotp{\ww'}{\phi(x)}=f_{\ww',H}(x)$.

 Since $f_{H,\ww}\in H_S$, we may write for some vector $\alpha$, \[f_{H,\ww}(x)= \sum_{i=1}^{n\choose p} \alpha_i k(\s{x}{i},x),\] and we obtain$\|\ww\|^2 = \sum \alpha_i\alpha_j k(\s{x}{i},\s{x}{j})$.

 By \cref{cor:vertices} there is a convex sum $\lambda_1,\ldots \lambda_{p+1}$ such that $k(\s{x}{i},\s{x}{j})= \sum \lambda_t k_{t,p}(\s{x}{i},\s{x}{j})$.
For each $t\le p$ set $\ww_t = \sum \alpha_i \phi_{t,p}(\s{x}{i}) \in H^p_t$, and define $\vv \in \universal_n^p$ to be \[\mathbf{v}=(\lambda_1(p+1) \ww_1,\ldots, \lambda_{p+1}(p+1)\ww_{p+1}).\]

Our proof is done if we can show that $\|\mathbf{v}\|\le (p+1)B$ and $f_{H,\ww}=f_{\universal_n^p,\mathbf{v}}$.  
First we show that  $f_{H,\ww}=f_{\universal_n^p,\mathbf{v}}$:
\begin{align*} f_{H,\ww}(x) &= \sum \alpha_i \sum \lambda_t k_{t,p}(\s{x}{i},x)\\
& = \sum \lambda_t \sum \alpha_i  k_{t,p}(\s{x}{i},x)\\
& = \sum  \dotp{\lambda_t(p+1)\cdot \ww_t}{\frac{1}{p+1}\phi_t(x)} \\
& =  \dotp{\vv}{\phi^u(x)}_{\universal_n} \\
&=f_{\mathbf{\universal_n^p},\vv}(x).\end{align*}
It remains to bound the norm of $\mathbf{v}$ by $(p+1) B$. First we obtain
\begin{align}\label{eq:wtbound}
B^2\ge \|\ww\|^2 &= \sum \alpha_i\alpha_j k(\s{x}{i},\s{x}{j}) \nonumber\\
& = \sum \alpha_i\alpha_j \sum \lambda_t k_{t,p}(\s{x}{i},\s{x}{j}) \nonumber\\
& = \sum \lambda_t \sum \alpha_i\alpha_j k_{t,p}(\s{x}{i},\s{x}{j})\nonumber\\
& = \sum \lambda_t \sum \|\ww_t\|^2 
\end{align}
Next, using \cref{eq:wtbound} we have that
\begin{align*}\|(\lambda_1(p+1)\ww_1,\ldots, \lambda_{p+1}(p+1)\ww_{p+1})\|_{\universal_n^p}&= \sqrt{\sum \|\lambda_t(p+1)\ww_t\|^2_{H_t}}\\&\le (p+1)\sqrt{\sum \lambda_t \|\ww_t\|^2}\le (p+1) B.\end{align*}

Finally, we address the computational issue of computing $k$. Note that to describe $k$ we need to solve the linear equations depicted in \cref{eq:lineq} and solve the linear equations $\Delta \s{\beta}{i} = e_i$. These equations can be solved in time $O(p^3)$. Once $\s{\beta}{i}$ are known we can compute (once) the function $g(k) = \sum \s{\beta}{i} {k\choose i}$ to compute $k(\s{x}{i},\s{x}{j})= g(\dotp{\s{x}{i}}{\s{x}{j}}$.
\ignore{
\begin{definition}[The Polytope $\cB(p)$]\label{def:cB}
Let $\cB(p) \subseteq \R^p$ be the polytope of all $\beta \in \R^p$ that satisfy the constraints of Lemma \ref{lem:characterization-regular-kernel}.
\end{definition}
Notice that by Lemma \ref{lem:characterization-regular-kernel}, membership in $\cB$ can be tested via a \textbf{linear} program in $\poly(p,n)$ time. 

\subsubsection{The Algorithm}
We are now ready to describe our algorithm for $\C_{p}(B).$ 

\newcommand{\algname}{Alg. 1 \xspace}
\begin{center}
\fbox{\begin{minipage}{6in} 
\begin{center}
\textbf{Algorithm \algname} 
\end{center}

\begin{description}
\item[Input:] For $m = $, $m$ i.i.d. samples from $\cD$ supported on $S_{p,n} \times \Y$: $\{(\s{x}{i}, y_i)\}_{i \leq m}$ and a loss function $\ell: \Y \times \Y \rightarrow \R$, convex and $1$-Lipschitz.

\item[Output: ] $\alpha \in \R^m$, $\beta \in \cB(p)$ defining the linear classifier $\sum_{i = 1}^m \alpha_i K_{\beta}(\s{x}{i}, \xx)$ in the Hilbert space associated with the kernel matrix $K_{\beta}$ defined by $K_\beta = \sum_{0 \leq \ell \leq p} \beta_{\ell} P_{p,\ell}.$ 
\item[Operation:] ~
\begin{enumerate}
\item Let $\ell^{*}$ be the Fenchel conjugate of the loss function $\ell$: $\ell^{*}(a,b) = \sup_{x} \langle a, x \rangle - \ell(x,b)$ for any $a,b.$
\item Set $P_{S,\ell}\in \mathbb{R}^{m\times m}$, be such that $P_{S,\ell}(i,j) = \kp_{p,\ell}(\s{x}{i}\cdot \s{x}{j})$.
\item Define $G_{S,\lambda}(\alpha, \beta) = -\frac{\lambda}{2} \sum_{0 \leq \ell \leq p} \beta_{\ell}  ( \alpha^{\top} P_{S,\ell} \alpha) - \frac{1}{m} \sum_{i = 1}^m \ell^{*}(\alpha_i/m, y_i).$ 
\item Set $\lambda = \Theta(\frac{\epsilon}{B^2}).$ 
\item\label{it:solve} Solve \[\mathop{\inf}_{\beta \in \cB(p)} \sup_{\alpha \in \R^{m}} G_{S,\lambda}(\alpha,\beta).\]
\item Output $\alpha,\beta$.
\end{enumerate}
\end{description}
\end{minipage}}
\end{center}

\subsubsection{Analysis: Running Time}
We analyze the running time and correctness of the algorithm in this section. 

The analysis of the algorithm is based on combining the analysis of the standard $\ell_2$-\regularized SVM algorithm with Lemma \ref{lem:characterization-regular-kernel}. We provide the details next.

For the running time upper bound, we only need to verify that Step \ref{it:solve} can be implemented efficiently. We show this next. 

\begin{lem}\label{lem:running-time}
There is an algorithm to compute $\inf_{\beta \in \cB(p)} \sup_{\alpha \in \R^{m}} G_{S,\lambda}(\alpha,\beta)$ in time $\poly(m,n)\log{(B/\epsilon)}$.

\end{lem}

\begin{proof}
$G_{S,\lambda}$ is linear (and thus convex) in $\beta$ for any fixed $\alpha$. We will write \[ G_{S,\lambda}(\beta) =\sup_{\alpha \in \R^p} G_{S,\lambda}(\alpha,\beta).\] Then $G_{S,\lambda}(\beta)$  is a supremum of convex functions and is thus convex in $\beta.$ At any $\beta$, one can efficiently compute $G_{S,\lambda}(\beta)$ by solving the concave program. Thus, it is enough to minimize $G_{S,\lambda}(\beta)$ as a function of $\beta$. 

To run any off-the-shelf convex minimization algorithm (such as the ellipsoid method), we only need to verify that we can also compute a subgradient of $G_{S,\lambda}(\beta)$ at any $\beta$ efficiently. It is a standard fact that if at any $\beta$ the supremum of a set of convex functions is achieved by one of the constituent functions, say, $G_{S,\lambda}(\beta)$ then, any subgradient of this constituent function is a subgradient of $G_{S,\lambda}$ at $\beta$. The latter is easy to compute given the explicit expression for $G_{S,\lambda}(\alpha, \beta)$ evaluated at the fixed $\beta$ and the optimizer $\alpha_1$ of $G_{S,\lambda}(\beta)$ at $\beta.$
 \end{proof}

Next, we show why minimizing $G_{S,\lambda}$ corresponds to learning the optimal linear classifier in any \regular RKHS for $S_{p,n}$. 

\begin{remark}
Similar facts have been observed before in the literature beginning with the influential work of Lanckriet et. al. \cite{lanckriet2004learning} (See Proposition 15).
\end{remark}

\begin{lem}\label{lem:primal}
Let $\beta \in \cB(p)$ define a RKHS $H_{\beta}$ for $S_{p,n}$ and given a sample $S\subseteq S_{p,n}$ consider the following minimization program: 
\begin{equation}
F_{S,\lambda}(\beta) = \inf_{\ww \in H_{\beta}} \frac{\lambda}{2} \|\ww\|_{H_{\beta}}^2 + \frac{1}{m} \cdot \sum_{i \leq m} \ell ( \langle \ww, \phi_{H_{\beta}}(\s{x}{i}), y_i).
\label{eq:primal}\end{equation} 
Then $F_{S,\lambda}$ is a convex function of $\beta$. In fact, $F_{S,\lambda}(\beta) = G_{S,\lambda}(\beta)$, and if $\beta^*$, $\alpha^*$ are the solution to $\inf \sup G_{S,\lambda}(\alpha,\beta)$ then the $\ww^*$ that minimizes the internal program in $F_{S,\lambda}(\beta^*)$ is given by
\[ \ww^* = \sum \alpha_i^* \phi_{\beta^*}(\s{x}{i}).\] 
\end{lem}
\begin{proof}
Given a sample $S$ and fixed $\beta$ denote by $K_{S,\beta}$ the kernel matrix obtained by $K_{S,\beta}(i,j)= k_{\beta}(\s{x}{i}{j},\s{x}{i}{j})$. Recall that we have similarly defined $P_{S,\ell}$ in \algname.
For a fixed $\beta$ by Fenchel's duality we can write
\begin{align*}
& \min_{\ww\in H_{\beta}} \frac{\lambda}{2}\|\ww\|^2_{H_{\beta}}+\frac{1}{m} \sum_{i=1}^m \ell(\dotp{\ww}{\phi_{H_{\beta}}(\s{x}{i})},y_i)=
\max_{\alpha \in \mathbb{R}^m} -\frac{\lambda}{2}\alpha^\top  K_{S,\beta} \alpha - \frac{1}{m}\sum_{i=1}^m \ell^* (\frac{\alpha_i}{m},y_i), 
\end{align*}
where $\ell^*(\alpha,y_i)= \max \alpha\cdot x -\ell(x,y_i)$ is the convex conjugate of $\frac{1}{m}\ell(x,y_i)$. Expanding $K_{S,\beta}= \sum \beta_j P_{S,j}$ we obtain:
\begin{align*}
 F_{S,\lambda}(\beta)=\sup_{\alpha} -\frac{\lambda}{2} \sum \beta_\ell \left(\alpha^\top P_{S,\ell} \alpha\right) - \frac{1}{m} \sum_{i=1}^m \ell^*(\frac{\alpha_i}{m},y_i)= 
\sup_{\alpha} G_{S,\lambda}(\alpha ,\beta) = G_{S,\lambda}(\beta)
\end{align*}
This establishes that convex program in Step \ref{it:solve} has the same optimum as the program in \eqref{eq:primal}. Let $\alpha^{*}, \beta^{*}$ be an optimum solution to $\inf_{\beta \in \cB(p)} \sup_{\alpha \in \R^{m}} G_{p,\lambda}$, by standard methods in SVM analysis (see \cite{shalev2014understanding} for example), one can in fact express $\langle \ww^{*}, \phi(\xx) \rangle$, the optimal linear classifier yielded by the primal program in terms of $\alpha^{*}$ and $\beta^{*}$ as: $\sum_{i\leq m} \alpha^{*}_{i} \phi_{\beta}(\s{x}{i})$. 
\end{proof}

\begin{corollary}\label{cor:rap-all}
Let $S=\{(\s{x}{i},y_i)\}_{i \leq m}\subseteq S_{p,n} \times \Y$ be a labeled sample. Set $\lambda = \Theta( \epsilon/B^2),$ and let $\alpha^*, \beta^*$ be the solution to
\begin{align*}
\inf_{\beta \in \cB(p)} \sup_{\alpha \in \R^{m}} G_{S,\lambda}(\alpha,\beta)\end{align*}

Set $k_{\beta^*} = \sum_{0 \leq \ell \leq p} \beta_{\ell} \kp_{p,\ell}$, then $k_{\beta^*}$ is a \regular kernel. 
Further, set
$\ww^* = \sum_{i=1}^m \alpha^*_i \phi_{H_\beta}(\s{x}{i})$ then $\|\ww^*\|_{H_\beta}=O(1/\sqrt{\lambda})$

Then, $$\loss_S(f_{H_{\beta^*}, \ww*}) \leq \inf_{f_{H,\ww} \in \C(B)} \loss_S(f_{H,\ww}) + \epsilon.$$ 

\end{corollary}

\begin{proof}
That $k_{\beta}$ is a kernel follows from $\beta \in \cB(p)$. Next by duality, one can show that $\ww^*$ is the minimizer of \eqref{eq:primal} for fixed $\beta^*$. It is easy to check that the minimizer $\ww$ of $F_{S,\lambda}(\beta)$ is a linear classifer that achieves error at most $\epsilon$ larger than the optimum error. Thus, if we could show that $\beta^*$ optimizes $F_{\lambda}(\beta)$ over all $\beta \in \cB(p)$, we'd be done. This is indeed what we proved in \lemref{lem:primal}.

Finally, we want to show that $\|\sum \alpha_i \phi_{\beta}(\s{x}{i})\|^2 = (\frac{1}{\lambda})$. As discused, by duality. the minimizer $\ww^*$ of the program in \eqref{eq:primal} is given by $\sum \alpha_i \phi_{\beta}(\s{x}{i})$. Note that $\ww=0$ attains the value $1$ hence the minimizer must attain less then $1$. But $\|\ww\|^2>\frac{1}{\lambda}$ is guaranteed to lead to loss greater then $1$ by the \regularization term. 
\end{proof}

\subsubsection{Analysis: Generalization}
We now bound the number of samples required in order for Algorithm 1 to succeed in generalizing. We adopt the standard technique of estimating the Rademacher complexity of the underlying class $\C_{p,n}(\frac{1}{\sqrt{\lambda}})$ and then using Fact \ref{fact:rad}. 

\begin{lem}[Rademacher Complexity of $\C_{p,n}(\frac{1}{\sqrt{\lambda}})$] \label{lem:rademacher-single-layer}
Let $S = \{\s{x}{1}, \s{x}{2}, \ldots, \s{x}{m}\} \subseteq S_{p,n}$. Then, $R_m(\C_{p,n}(\frac{1}{\sqrt{\lambda}}),S) \leq \sqrt{\frac{\Theta(\log{(p)})}{\sqrt{\lambda} m}}.$
\end{lem}
\begin{proof}
Our idea is to appeal to Fact \ref{fact:cortes}. Thus, we will establish that every \regular kernel can be written as a convex combination of at most $p+1$ ``base'' kernels. Consider the matrix $\Pi_i$ for $0 \leq  i \leq p$, that is the projector from $\R^{{n \choose p}} \rightarrow V_i$, the $i^{th}$ eigenspace of $\johnson_{n,t}.$ Then, using Fact \ref{fact:eigendecomp-johnson}, every matrix in the Johnson scheme (and thus, the kernel matrix of any \regular kernel $K$) can be written as a non-negative combination of $\Pi_i$ as $K = \lambda_i \Pi_i$ where $\lambda_i$ is the eigenvalue of $K$ on $V_i$. Applying Fact \ref{fact:cortes} to the kernels with kernel matrices $c_i \Pi_i$ (where $c_i$ is a scaling chosen to make $\Pi_i$ have maximum entry on the diagonal $1$), we obtain the stated bound on the Rademacher complexity of $\C_{p,n}(B)$. 
\end{proof}}


\subsection{Putting it all together}\label{sec:mainproof}

By \cref{lem:sp} there are RKHS $k_1,\ldots,k_n$ associated with RKHS $((\universal_n^1,\phi^n_1),\ldots, (\universal^n_n,\phi^n_n)$, such that $\C_{p,n}(B)\subseteq U_n^p((n+1)B)$. Each kernel can be computed using a preprocess stage of $O(n^3)$, overall we can compute the whole class of kernels in time $O(n^4)$, then the computation of each entry of the kernel take times $T=O(n)$. \cref{lem:mainreduction} then says that there exists a universal RKHS $(\universal_n,\phi_n)$ such that $\C_{\oplus_n}(B)\subseteq\universal_n(nB)$.

Finally we obtain by \cref{lem:impropermkl} that \[\C_n(B)\subseteq \C_{\oplus_n}(\sqrt{n}B)\subseteq \universal_n((n+1)\sqrt{n}B).\]

The computation of each entry in the kernel is than given by $O(n+T(n))=O(n)$.
\section{Proof of \cref{thm:mkl}}\label{prf:mkl}
We next restate \cref{thm:mkl} which we prove in this section.
\mkl*
Similar to \cref{thm:main} our idea is to return a function $f_{H,\ww}\in \C_{\oplus_n}(\sqrt{n}B)\subseteq \C(\sqrt{n}B)$ and reduce the problem to the single hyper cube layers. Unlike \cref{thm:main}, to learn over the layers, we will not construct a universal kernel, but instead we will apply the tools from Multiple Kernel Learning, to output a target function $f_{H,\ww} \in \H_{\C_{p,n}}$ that optimizes over the regulerized objective. This is the procedure that helps us in shaving off a factor $n$ in the sample complexity. Concretely, we will develop an efficient algorithm for the following optimization problem:

\begin{align}\label{eq:spmkl} \mathop\mathrm{minimize}_{\{(\ww,H)\mid H\in \C_{p,n}, \ww\in H\}} &~\frac{\lambda}{2}\|\ww\|^2 +\mathcal{L}_S(f_{H,\ww})
\end{align}

We will then proceed to derive generalization bounds for the class $\C_{\oplus_n}(B)$. The final details of the proof are then summed up in \cref{sec:prfmkl}.

\subsection{Reduction to the hypercube layer $S_{p,n}$}

We next set out to learn a regulerized objective over $\H_{\C_{\oplus_n}}$:
\begin{lem}\label{lem:spreduction2}
For every $n$, let $S=\{(\s{\xx}{i},y_i)\}_{i=1}^m$ be a sample from $\X_n$. Suppose that for every sample $S\subseteq S_{p,n}$ there exists an efficient algorithm that runs in time $T(n,|S|,1/\epsilon)$ and solves the optimization problem in \cref{eq:spmkl} up to $\epsilon$ error.
Then the following optimization problem can be solved efficiently in time $n T(n, |S|,n/\epsilon)$ to $\epsilon$ accuracy. 

\begin{align}\label{eq:lambda} \mathop\mathrm{minimize}_{\{(\ww,H)\mid H\in \H_{\C_{\oplus_n}}, \ww\in H\}} &~\frac{\lambda}{2}\|\ww\|_H^2 + \mathcal{L}_{S}(f_{H,\ww})
\end{align}
\end{lem}
\begin{proof}
By the structure of $\H_{\C_{\oplus_n}}$ we can write 
\begin{align*} &\min_{\{(\ww,H)\mid H\in \H_{\C_{\oplus_n}}, \ww\in H\}} \frac{\lambda}{2}\|\ww\|^2 + \sum_{i=1}^m \ell(\dotp{\ww}{\phi(\s{x}{i}}_{H},y_i) \\
=& \min_{\{(\ww_1,\ldots,\ww_n),H_1\oplus H_2\oplus\cdots\oplus H_n)\mid H_p\in H_{\C_{p,n}}, \ww_p\in H_p\}} \frac{\lambda}{2}\sum_{p=1}^n \|\ww_p\|_{H_p}^2 + \sum_{p=1}^n \sum_{\s{\xx}{i}\in S_{p,n}}\ell(\dotp{\ww_p}{\phi(\s{x}{i}}_{H},y_i)\\
=& \sum_{p=1}^n \min_{\{(\ww_p),H_p)\mid H_p\in \H_{\C_{p,n}}, \ww_p\in H_p\}} \frac{\lambda}{2} \|\ww_p\|_{H_p}^2 +\sum_{\s{\xx}{i}\in S_{p,n}}\ell(\dotp{\ww_p}{\phi(\s{x}{i}}_{H},y_i) 
\end{align*}
By assumption we can now solve each $n$ optimization problems in the summands efficiently to obtain an optimal $\ww=(\ww_1,\ldots, \ww_n)$ and an RKHS $H=H_1\oplus\ldots\oplus H_n$.
\end{proof}
\subsection{Efficient algorithm for learning $\H_{\C_{p,n}}$}
Our next step in the proof relies on proposing an efficient optimization algorithm over class $\H_{\C_{p,n}}$. In contrast with previous section, we will not relax the task of learning $\C_{p,n}(B)$ and propose an improper formulation. Instead we directly optimize over the kernel and linear separator using tools from MKL. The main result for this section is the following Lemma, which is proved at the end.
\begin{lem}\label{lem:efficient}
For every $p$, let $S=\{(\s{\xx}{i},y_i )\}_{i=1}^m$ be a sample from $S_{p,n}$. The optimization problem in \cref{eq:spmkl} can be solved efficiently in time $\mathrm{poly}(\frac{1}{\lambda},1/\epsilon,m)$ to $\epsilon$ accuracy. 

\end{lem}
The proof utilizes the convexity of the program that can be demonstrated by duality-- this observation has been made and exploited for MKL in \cite{lanckriet2004learning} and followups. The second ingredient of the proof uses the nice structure of the class of \regular kernels over $S_{p,n}$ which are defined by $(p+1)$ linear constraints. For a general class of kernel matrices, MKL may involve adding a semi-positiveness constraint which may turn the problem into a non-scalable SDP. Here however, the nice structure of \regular kernels, gives us a tractable representation over a convex sum of few base kernels.

To describe the algorithm we add further notations: First let us denote by \[\mathcal{B}(p+1)= \{\beta \in \mathbb{R}^{p+1}\mid~ \beta\ge 0 \sum \beta_i\le 1\}\] the $p+1$ dimensional simplex and for each $\beta\in \mathcal{B}(p+1)$ we write $k_{\beta}$ to denote the kernel $k_{\beta} = \sum \beta_i k_{p,i}$, where $k_i$ are as given in \cref{cor:vertices}. Note that $k$ is a kernel if and only if $k=k_{\beta}$ for some $\beta \in \mathcal{B}(p+1)$. We will similarly denote by $(H_{\beta},\phi_{\beta})$ the associated RKHS. We next describe the algorithm for solving \cref{eq:primal}

.
\begin{center}
\fbox{\begin{minipage}{6in} 
\begin{center}
\textbf{Algorithm \algname} 
\end{center}

\begin{description}
\item[Input:] For $m = $, $m$ i.i.d. samples from $\cD$ supported on $S_{p,n} \times \Y$: $\{(\s{x}{i}, y_i)\}_{i \leq m}$ and a loss function $\ell: \Y \times \Y \rightarrow \R$, convex and $1$-Lipschitz.

\item[Output: ] $\alpha \in \R^m$, $\beta\in \cB(p+1)$ defining the linear classifier $\sum_{i = 1}^m \alpha_i K_{\beta}(\s{x}{i}, \xx)$ in the Hilbert space associated with the kernel matrix $K_{\beta}$ defined by $K_\beta = \sum_{0 \leq t \leq p} \beta_{t} k_{p,t}.$ where $k_{p,t}$ are given by \cref{cor:vertices}. 
\item[Operation:] ~
\begin{enumerate}
\item Let $\ell^{*}$ be the Fenchel conjugate of the loss function $\ell$: $\ell^{*}(a,b) = \sup_{x} \langle a, x \rangle - \ell(x,b)$ for any $a,b.$
\item For $0\le t\le p$ set $K_{S,t}\in \mathbb{R}^{m\times m}$, be such that $K_{S,t}(i,j) = k_{p,t}(\s{x}{i}\cdot \s{x}{j})$.
\item Define $G_{S,\lambda}(\alpha, \beta) = -\frac{\lambda}{2} \sum_{0 \leq t \leq p} \beta_{t}  ( \alpha^{\top} K_{S,t} \alpha) - \frac{1}{m} \sum_{i = 1}^m \ell^{*}(\alpha_i/m, y_i).$ 
\item\label{it:solve} Solve \[\mathop{\inf}_{\beta\in\cB(p+1)} \sup_{\alpha \in \R^{m}} G_{S,\lambda}(\alpha,\beta).\]
\item Output $\alpha,\beta$.
\end{enumerate}
\end{description}
\end{minipage}}
\end{center}

\subsection{Analysis: Running Time and Correctness}
We analyze the running time and correctness of the algorithm in this section. 

The analysis of the algorithm is based on combining the analysis of the standard $\ell_2$-regularized SVM algorithm with Lemma \ref{cor:vertices}. We provide the details next.

For the running time upper bound, we only need to verify that Step \ref{it:solve} can be implemented efficiently. We show this next. 

\begin{lem}\label{lem:running-time}
There is an algorithm to compute $\inf_{\beta \in \mathcal{B}(p+1)} \sup_{\alpha \in \R^{m}} G_{S,\lambda}(\alpha,\beta)$ in time $\poly(m,n)\log{(B/\epsilon)}$.

\end{lem}

\begin{proof}
$G_{S,\lambda}$ is linear (and thus convex) in $\beta$ for any fixed $\alpha$. We will write \[ G_{S,\lambda}(\beta) =\sup_{\alpha \in \R^p} G_{S,\lambda}(\alpha,\beta).\] Then $G_{S,\lambda}(\beta)$  is a supremum of convex functions and is thus convex in $\beta.$ At any $\beta$, one can efficiently compute $G_{S,\lambda}(\beta)$ by solving the concave program. Thus, it is enough to minimize $G_{S,\lambda}(\beta)$ as a function of $\beta$. 

To run any off-the-shelf convex minimization algorithm, we only need to verify that we can also compute a subgradient of $G_{S,\lambda}(\beta)$ at any $\beta$ efficiently. It is a standard fact that if at any $\beta$ the supremum of a set of convex functions is achieved by one of the constituent functions, say, $G_{S,\lambda}(\beta)$ then, any subgradient of this constituent function is a subgradient of $G_{S,\lambda}$ at $\beta$. The latter is easy to compute given the explicit expression for $G_{S,\lambda}(\alpha, \beta)$ evaluated at the fixed $\beta$ and the optimizer $\alpha_1$ of $G_{S,\lambda}(\beta)$ at $\beta.$
 \end{proof}

Next, we show why minimizing $G_{S,\lambda}$ corresponds to learning the optimal linear classifier in any regular RKHS for $S_{p,n}$. 

\begin{remark}
Similar facts have been observed before in the literature beginning with the influential work of Lanckriet et. al. \cite{lanckriet2004learning} (See Proposition 15).
\end{remark}

\begin{lem}\label{lem:primal}
Let $\beta\in \cB(p+1)$ define a RKHS $H_{\beta}$ for $S_{p,n}$ and given a sample $S\subseteq S_{p,n}$ consider the following minimization program: 
\begin{equation}
F_{S,\lambda}(\beta) = \inf_{\ww \in H_{\beta}} \frac{\lambda}{2} \|\ww\|_{H_{\beta}}^2 + \frac{1}{m} \cdot \sum_{i \leq m} \ell ( \dotp{\ww}{\phi_{H_{\beta}}(\s{x}{i})}, y_i).
\label{eq:primal}\end{equation} 
Then $F_{S,\lambda}$ is a convex function of $\beta$. In fact, $F_{S,\lambda}(\beta) = G_{S,\lambda}(\beta)$, and if $\beta^*$, $\alpha^*$ are the solution to $\inf \sup G_{S,\lambda}(\alpha,\beta)$ then the $\ww^*$ that minimizes the internal program in $F_{S,\lambda}(\beta^*)$ is given by
\[ \ww^* = \sum \alpha_i^* \phi_{\beta^*}(\s{x}{i}).\] 
\end{lem}
\begin{proof}
Given a sample $S$ and fixed $\beta\in \cB(p+1)$ denote by $K_{S,\beta}$ the kernel matrix obtained by $K_{S,\beta}(i,j)= k_{\beta}(\s{x}{i},\s{x}{j})$. Recall that we have similarly defined $K_{S,t}$ for $0\le t \le p+1$ in \cref{fig:alg}.
For a fixed $\beta$ by Fenchel's duality we can write
\begin{align*}
& \min_{\ww\in H_{\beta}} \frac{\lambda}{2}\|\ww\|^2_{H_{\beta}}+\frac{1}{m} \sum_{i=1}^m \ell(\dotp{\ww}{\phi_{H_{\beta}}(\s{x}{i})},y_i)=
\max_{\alpha \in \mathbb{R}^m} -\frac{\lambda}{2}\alpha^\top  K_{S,\beta} \alpha - \frac{1}{m}\sum_{i=1}^m \ell^* (\frac{\alpha_i}{m},y_i), 
\end{align*}
where $\ell^*(\alpha,y_i)= \max \alpha\cdot x -\ell(x,y_i)$ is the convex conjugate of $\frac{1}{m}\ell(x,y_i)$. Expanding $K_{S,\beta}= \sum \beta_t K_{S,t}$ we obtain:
\begin{align*}
 F_{S,\lambda}(\beta)=\sup_{\alpha} -\frac{\lambda}{2} \sum \beta_t \left(\alpha^\top K_{S,t} \alpha\right) - \frac{1}{m} \sum_{i=1}^m \ell^*(\frac{\alpha_i}{m},y_i)= 
\sup_{\alpha} G_{S,\lambda}(\alpha ,\beta) = G_{S,\lambda}(\beta)
\end{align*}
This establishes that convex program in Step \ref{it:solve} has the same optimum as the program in \eqref{eq:primal}. Let $\alpha^{*}, \beta^{*}$ be an optimum solution to $\inf_{\beta \in \cB(p)} \sup_{\alpha \in \R^{m}} G_{p,\lambda}$, by standard methods in SVM analysis (see \cite{shalev2014understanding} for example), one can in fact express $\langle \ww^{*}, \phi(\xx) \rangle$, the optimal linear classifier yielded by the primal program in terms of $\alpha^{*}$ and $\beta^{*}$ as: $\sum_{i\leq m} \alpha^{*}_{i} \phi_{\beta}(\s{x}{i})$. 
\end{proof}
\paragraph{Proof of \cref{lem:efficient}}
The proof is an immediate corollary of \cref{lem:primal} and the structure of $\H_{\C_{p,n}}$ depicted in \cref{cor:vertices}.
\ignore{
\subsection{Choosing $\lambda$:}

Combining \cref{lem:spreduction2,lem:efficient} we obtain that we can efficiently solve the optimization problem depicted in \cref{eq:lambda}. Namely, there is an efficient algorithm that outputs a $\epsilon$--accurate solution to the following program parametrized by $\lambda$:
\begin{align*}\label{eq:lambda2} \mathop\mathrm{minimize}_{\{(\ww,H)\mid H\in \H_{\C_{\oplus_n}}, \ww\in H\}} &~\frac{\lambda}{2}\|\ww\|_H^2 + \mathcal{L}_{S}(f_{H,\ww})
\end{align*}

For each $\lambda$ denote by $\ww(\lambda)$ the optimizer in \cref{eq:lambda} and by $B(\lambda):=\|\ww(\lambda)\|$, similarly define $H(\lambda)$.

Note that if we knew $\lambda^*$ such that $B(\lambda^*)=B$ we are done (up to proving a generalization bound). Indeed if the output solution has norm $\|\ww(\lambda)\|=B$, then clearly for every $\|\ww\|_{H}\le B$ we have that $\mathcal{L}_S(f_{H,\ww})\ge \mathcal{L}_S(f_{H(\lambda),\ww(\lambda)})$. Hence $\ww(\lambda)$ is indeed the minimizer for the empirical risk. In reality though, we do not have the correct $\lambda$ thus we will invoke binary search. In this section we discuss the continuity of the norm of the solution in terms of $\lambda$ to justify a binary search. Namely, we wish to prove the following statement.
\begin{lemma}
Let $0<\tau \le \lambda_1,\lambda_2 \le \sqrt{2}$ be such that $|\lambda_1-\lambda_2|\le \epsilon$, and denote by $\ww_i=\ww(\lambda_i)$ and $H_i=H(\lambda_i)$ for $i=1,2$. Then
\[\|\ww_1\|_{H_1}-\|\ww_2\|_{H_2}\le \epsilon\]
\end{lemma}

\begin{proof}
First, consider the Hilbert space $H=H_1\oplus H_2$ with the embedding $\phi(x)=\frac{1}{2}  (\phi_1(x)\oplus \phi_2(x))$. Then one can observe that $\lambda_1$ and $\lambda_2$ are the minimizers of the renormalized objective with fixed known Hilbert space $H$:

\[F(\lambda) = \min_{\{\ww\in H\}} \lambda\|\ww\|_H^2 + \mathcal{L}_{S}(f_{H,\ww})\]
\end{proof}

Since $\mathcal{L}_S(f_{H,0})\le 1$ by assumption we have that $B(1)\le \sqrt{2}$. Indeed any $\|\ww\|$ with norm strictly larger than $\sqrt{2}$ cannot be optimal as it performs worse than $\ww=0$. 

We next assume, w.l.o.g, that $B\ge \sqrt{2}$. Denote by $\lambda^* < 1$ the assignment $\lambda$ such that $B(\lambda^*)=B$. In reality we cannot compute $\lambda^*$ however note that the function $B(\lambda)$ is monotonically decreasing and using binary search and solving \cref{eq:lambda}, $O\left(\log (\frac{\epsilon}{B^2})\right)$ times, we can compute a $\lambda_0\ge \lambda^*$ such that $B(\lambda_0) \le B$ and $|\lambda_0-\lambda^*|<\frac{\epsilon}{B^2}$. 
Let $\ww_0$ be the solution to \cref{eq:lambda} with $\lambda=\lambda_0$ then we have the following:

\begin{align*}
\mathcal{L}_S(\ww_0) 
&\le \frac{\lambda_0-\lambda^*}{2} \|\ww_0\|^2+ \mathcal{L}_S(\ww_0)\\
&\le \frac{\lambda_0-\lambda^*}{2} \|\ww^*\|^2+ \mathcal{L}_S(\ww^*)\\
&\le \frac{\lambda_0-\lambda^*}{2} B^2+ \mathcal{L}_S(\ww^*)\\
&=\epsilon + \mathcal{L}_S(\ww)
\end{align*}
Thus, $\ww_0$ gives an $\epsilon$ close solution to the empirical minimization error and $\|\ww\|\le B$, namely $f_{H,\ww}\in \bar\C(B)$. Finally, using \cref{lem:generalization} we obtain that if the sample size is $m= O(\frac{B^2}{\epsilon^2}\log n)$ the solution will generlize and we $\ww_0$ minimizes (up to $\epsilon$ error) the generalization error. 
\end{proof}}
\subsection{Generalization bounds for the class $\C_{\oplus_n}(B)$}
We next set out to prove the following generalization bound for learning the class $\C_{\oplus_n}(B)$
\begin{lem}\label{lem:generalization}
Let $\ell$ be a Lipschitz convex loss function. Given an IID sample $S$ of size $m$ from an unknown distribution $\mathcal{D}$ supported over $\X_n\times \Y$. With probability $2/3$ the following holds for every $f_{H,\ww} \in \C_{\oplus_n}(B)$ (uniformly)

\begin{align*}
\mathcal{L}_{S}(f_{H,\ww}) \le \mathcal{L}_{D}(f_{H,\ww}) + O\left(B\sqrt{\frac{\log n}{S}}\right)
\end{align*}
\end{lem}
The proof relies on the following bound on the Rademacher complexity and the following standard generalization bound:
Recall that the Rademacher Complexity of a class $\H$ over a sample $S=\{\s{x}{1},\ldots, \s{x}{m}\}$ is defined as follows
\[R_m(\H,S) = \mathbb{E}_{\sigma} \left[ \sup_{f\in \H}   \frac{1}{m}\sum_{i=1}^m \sigma_{i} f(\s{x}{i})\right]\]
where $\sigma \in \{-1,1\}^t$ are i.i.d.~Rademacher distributed random variables.
The following bound the generalization performance of an empirical risk minimizer with respect to the class $\H$. (e.g. \cite{shalev2014understanding, bartlett2002rademacher})

\begin{fact} \label{fact:rad}
Let $\ell$ be a $1$-Lipschitz convex loss function with $|\ell(0,y)|\le 1$. Assume that for all $\xx$ and $f\in \H$ we have $|f(\xx)|<c$.  Given an IID sample from $\cD$ supported over $\X \times \Y$, for any $f\in \H$ with probability at least $1-\delta$ (over $S$):

\begin{align}\label{eq:radbound}
\loss_{S_m}(f) \le  \loss_{\cD}(f) + 4 \sup_{S_m} R_m(\H,S)+ 4c\sqrt{\frac{2\ln 2/\delta}{m}}
\end{align}
\end{fact}

\begin{lem}
For the class $\C_{\oplus_n}(B)$, we have the following bound on the Rademacher Complexity

\[ \rad(\C_{\oplus_n(B)},S) \le \sqrt{\frac{2e B^2\log n}{|S|}} \]
where e is the natural exponent $e = \lim_{n\to \infty}(1-\frac{1}{n})^n$.
\end{lem}

\begin{proof}
For each $p$ and sample $S$ denote $S_p =S\cap \{\s{\xx}{i}\mid \sum \s{\xx}{i}=p\}$, and recall that for every $f_{\ww,H}\in \C_{\oplus_n}$ we can write $\ww= \ww_1\oplus\ww_2\oplus\cdots \oplus _n$ where $\ww_{p}\in \C_{p,n}(\|\ww_p\|)$ and $\sum \|\ww_p\|^2 \le B$.

By definition of the Rademacher Complexity we have the following:

\begin{align*}
|S|\cdot \rad(\C_{\oplus_n(B)},S) &= \EE{\sup_{f_{\ww,H}\in \C_{\oplus_n}(B)} \sum_{\phi(\s{x}{i})\in S} \sigma_i f_{\ww,H}(\phi(\s{x}{i}))}\\
&= \EE{\sup_{f_{\ww,H}\in \C_{\oplus_n}(B)} \sum_{p=1}^n \sum_{\phi(\s{x}{i})\in S_p} \sigma_i f_{\ww,H}(\phi(\s{\xx}{i}))} \\
&= \EE{\sup_{\{\sum B_p^2 \le B\}} \sum_{p=1}^n\sup_{f_{\ww_p,H_p}\in \C_{p,n}(B_p)} \sum_{\phi(\s{x}{i})\in S_p} \sigma_i f_{\ww_p,H_p}(\phi(\s{\xx}{i}))}\\
&= \EE{\sup_{\{\sum B_p^2 \le B\}} \sum_{p=1}^n\sup_{f_{\ww_p,H_p}\in \C_{p,n}(B_p)} {\langle\ww_p;\sum_{\phi(\s{\xx}{i})\in S_p} \sigma_i\phi(\s{\xx}{i})}\rangle_{H_p}}
\end{align*}
Note that by letting $\ww_p = \sum_{\phi(\s{\xx}{i})\in S_p} \sigma_i\phi(\s{\xx}{i})$ and by Cauchy Schwartz we have that \[\sup_{\|\ww_p\|\le B_p}{\langle\ww_p;\sum_{\phi(\s{\xx}{i})\in S_p} \sigma_i\phi(\s{\xx}{i})}\rangle_{H_p}= B_p\|\sum_{\phi(\s{\xx}{i})\in S_p} \sigma_i\phi(\s{\xx}{i})\|_{H_p}\]
Thus we continue with the derivation and obtain
\begin{align*} \EE{\sup_{\{\sum B_p^2 \le B\}} \sum_{p=1}^n\sup_{f_{\ww_p,H_p}\in \C_{p,n}(B_p)} {\langle\ww_p;\sum_{\phi(\s{\xx}{i})\in S_p} \sigma_i\phi(\s{\xx}{i})}\rangle_{H_p}}
=\EE{\sup_{\{\sum B_p^2 \le B\}} \sum_{p=1}^n B_p \sup_{H_p\in H_{\C_{p,n}}} {\|\sum \sigma_i \phi(\s{\xx}{i})\|_{H_p}}}
\end{align*}
Again we apply C.S inequality to choose $B_p \propto  \sup_{H_p\in H_{\C_{p,n}}} {\|\sum \sigma_i \phi(\s{\xx}{i})\|_{H_p}}$. and obtain 
\begin{align*}
|S|\cdot\rad(\C_{\oplus_n(B), S}) &=\EE{\sup_{\{\sum B_p^2 \le B\}} \sum_{p=1}^n B_p \sup_{H_p\in H_{\C_{p,n}}} {\|\sum \sigma_i \phi(\s{\xx}{i})\|_{H_p}}}\\
&=\EE{B\sqrt{ \sum_{p=1}^n \left(\sup_{H_p\in H_{\C_{p,n}}} {\|\sum \sigma_i \phi(\s{\xx}{i})\|_{H_p}}\right)^2}} \\
&\le  B\sqrt{ \sum_{p=1}^n \EE{ \left(\sup_{H_p\in H_{\C_{p,n}}} {\|\sum \sigma_i \phi(\s{\xx}{i})\|_{H_p}}\right)^2}}& \textrm{Concavity of $\sqrt{}$}
\end{align*}
We next set out to bound  the quantity $\EE{ \left(\sup_{H_p\in H_{\C_{p,n}}} {\|\sum \sigma_i \phi(\s{\xx}{i})\|_{H_p}}\right)^2}$. At this step our proof follows the foots steps of \cite{cortes2010generalization} who bound a similar quantity for achieving their generalization bound. First recall that $H_{\mathcal{J}_{p,n}}$, consists of all Hilbert spaces induced by taking as a kernel the convex hull of the Hilbert spaces that we will denote $H_{p,1},\ldots, H_{p,p+1}$. One can then show that 
\begin{align*}\sup_{H_p\in H_{\C_{p,n}}} {\|\sum \sigma_i \phi(\s{\xx}{i})\|^2_{H_p}} &= \sup_{k} \|\sum \sigma_i \phi(\s{x}{i})\|^2_{H_{p,k}}\\
& \le \left(\sum_{k=1}^{p+1} \|\sum \sigma_i \phi(\s{\xx}{i})\|^{2r}_{H_p,k}\right)^{1/r}, &\forall r\ge 1 \\
&= \left(\sum_{k=1}^{p+1} \left(\sigma^\top K_{p,k} \sigma\right)^r\right)^{1/r}
\end{align*}
By concavity we then obtain
\begin{align*}
 \EE{ \left(\sup_{H_p\in H_{\C_{p,n}}} {\|\sum \sigma_i \phi(\s{\xx}{i})\|^2_{H_p}}\right)}&\le \left(\sum_{k=1}^{p+1}\EE{\left( \sigma^\top K_{p,k} \sigma\right)^r}\right)^{1/r}
\end{align*}
By Lemma 1 in \cite{cortes2010generalization}, we have the following inequality 
\[\EE{\left( \sigma^\top K_{p,k} \sigma\right)^r} \le (2 r \mathrm{Tr}(K_{p,k}))^r\]
Also, since $\mathrm{Tr}(K_{p,k}) \le |S_p|$ we obtain that for all $r\ge 1$
\begin{align*}
\left(\sum_{k=1}^{p+1}\EE{\left( \sigma^\top K_{p,k} \sigma\right)^r}\right)^{1/r}&\le (p(2r|S_p|)^r)^{1/r}& \textrm{Set $r=\log p$} \\
& =(2e(\log p |S_p|))
\end{align*}

Overall we obtain that
\begin{align*}
|S|\rad(\C_{\oplus_n}(B),S)&\le B \sqrt{\sum_{p=1}^n (2e (\log p |S_p|)) }  \\
&\le  B \log n\sqrt{2e \sum_{p=1}^n |S_p|\log n} \\
&= B \sqrt{2e|S|\log n} 
\end{align*}
\end{proof}

\subsection{Putting it all together}\label{sec:prfmkl}

Consider the optimization problem in \cref{eq:lambda} with $\lambda=\frac{\epsilon}{nB^2}$. Note that by assumption that $\ell$ is bounded by $1$ at $\ww=0$ we have in particular that the minimizer obtain an objective smaller than $1$ (which is the objective obtained by $\ww=0$. In particular if $f_{H^*, \ww^*}$  minimizes \cref{eq:lambda} up to $\frac{\epsilon}{2}$ error then $\|\ww\|\le \sqrt{\frac{n}{\epsilon}}B$, and hence $f_{H^*,\ww^*}\in \C_{\oplus_n}(\sqrt{\frac{n}{\epsilon}B})$. Also for every solution $f_{H,\ww}\in C_{\oplus_n}(\sqrt{n} B)$, using the generalization bound in \cref{lem:generalization} we obtain that w.p. $2/3$, if $S=O(nB \sqrt{\frac{\log n}{m}})$:
\begin{align*}
\mathcal{L}_{D}(f_{H^*,\ww^*})&\le \mathcal{L}_{S}(f_{H^*,\ww^*}) +\epsilon \\
&\le \frac{\lambda}{2}\|\ww^*\|^2 +\mathcal{L}_{S}(f_{H^*,\ww^*}) +\epsilon \\
& \le \min_{f_{H,\ww}\in \C_{\oplus_n}(B)}\frac{\lambda}{2}\|\ww\|^2 +\mathcal{L}_{S}(f_{H,\ww}) +\epsilon \\
&\le \min_{f_{H,\ww}\in \C_{\oplus_n}(B)}\mathcal{L}_{S}(f_{H,\ww}) +2\epsilon\\
&\le  \min_{f_{H,\ww}\in \C_{\oplus_n}(B)}\mathcal{L}_{D}(f_{H,\ww}) +3\epsilon
\end{align*}
\section{Proof of \cref{thm:solid}} \label{prf:solid}
\solid*
In this section, we extend the algorithm from previous sections to arbitrary distributions with marginals supported over the solid hypercube $[0,1]^n \subseteq \R^n$. This captured kernel learning over any bounded subset of $\R^n$ up to rescaling.

\ignore{
for every $\s{x}{1},\s{x'}{1},\s{x}{2},\s{x'}{2} \in [0,1]^n$, $|k(\s{x}{1},\s{x}{2}) - k(\s{x'}{1},\s{x'}{2})| \leq L \cdot (\|\s{x}{1}-\s{x'}{1}\| + \|\s{x}{2}-\s{x'}{2}\|).$}
\ignore{
\begin{remark}
All \regular kernels over $[0,1]^n$ used in the literature are Lipchitz continuous with reasonably small Lipschitz constants $L$. For example, the (normalized) degree $d$ polynomial kernel $k(x,y) = \frac{1}{2^d}(\langle x, y \rangle +1)^d$ is $1$-Lipschitz continuous and the Gaussian radial basis function $k(x,y) = e^{-\frac{\|x-y\|_2^2}{\sigma}}$ is $\frac{1}{\sigma}$-Lipschitz continuous.
\end{remark}}

Our idea is essentially discretization of the solid hypercube in order to view it as a hypercube in a somewhat larger dimension. We thus define the following useful object.

\begin{definition}[$\epsilon$-Hypercube Embedding]
Fix an $\epsilon >0$. A pair of functions $\{\Psi_1,\Psi_2\}:[0,1]^n \rightarrow \zo^{nt}$ is said to be an $\epsilon$-Hypercube pair embedding of the unit cube in $nt$ dimensions, if for every $\s{x}{1}, \s{x}{2} \in [0,1]^n$:
$|\langle \s{x}{1},\s{x}{2} \rangle - \frac{1}{t} \langle \Psi_1(\s{x}{1}), \Psi_2(\s{x}{2}) \rangle| \leq \epsilon.$
\end{definition}

It is easy to construct $\epsilon$-Hypercube pair embeddings of $[0,1]^n$. We start with an embedding of the unit interval as given by the following lemma.
\begin{lem}[$\epsilon$-Hypercube Embedding of the Unit Interval] \label{lem:embedding-interval}
Fix an $\epsilon > 0$. There exists a $t = \Theta(\log{\frac{1}{\epsilon}}/\epsilon^2)$ and an efficiently computable randomized maps  $\psi_i:[0,1] \rightarrow \zo^{t}$ such that for any $x_1,x_2\in [0,1]$, $| x_1x_2 - \frac{1}{t} \langle \psi_1(x_1), \psi_2(x_2) \rangle |\leq 2\epsilon $.
\end{lem}
\begin{proof}
Let $\bar{x}$ for any $x \in [0,1]$ denote the value obtained by rounding down to the nearest multiple of $\epsilon/3$. Then, notice that $|x_1x_2-\bar{x_1}\bar{x_2}| \leq \epsilon.$ Next, for every $\bar{x}$, choose $\psi_i({\bar{x}}) \in \zo^t$ by setting $\psi_i({\bar{x}})_j$ independently with probability $\bar{x}$ to be $1$ and $0$ otherwise. Then, notice that $\E[\langle \psi_1({\bar{x_1}}), \psi_2({\bar{x_2}}) \rangle] = t\bar{x_1} \bar{x_2}$. Further,  for any fixed $\bar{x_1}, \bar{x_2}$, $\Pr[ |\langle \psi_1({\bar{x_1}}), \psi_2({\bar{x_2}}) \rangle - t\bar{x_1} \bar{x_2}| > t \epsilon] \leq \epsilon^2/100$ for some $t = \Theta(\log{\frac{1}{\epsilon}}/\epsilon^2)$. By a union bound, for every $\bar{x_1}, \bar{x_2}$ in the discretized interval $[0,1]$, we have: $|\langle \psi_1({\bar{x_1}}), \psi_2({\bar{x_2}}) \rangle - t\bar{x_1} \bar{x_2}| \leq t \epsilon$ with probability at least $2/3$ as required.
\end{proof}

We can now use \cref{lem:embedding-interval} to obtain an $\epsilon$-Hypercube Embedding of $[0,1]^n$.

\begin{lem}[$\epsilon$-Hypercube Embedding of the Unit Ball]
For any $\epsilon > 0$, there's an efficiently computable explicit randomized map that with probability at least $2/3$ outputs an $\epsilon$-Hypercube Embedding of $[0,1]^n$, with $t= O(\frac{n^2}{\epsilon^2}\log \frac{n}{\epsilon})$.
\end{lem}

\begin{proof}
Let $\psi_i$ be a pair of $\epsilon/n$-Hypercube Embedding of the unit interval in $t$ dimensions. Let $\Psi_i:[0,1]^n\rightarrow \zo^{nt}$ be defined as $\Psi_i(\xx) = \psi_i^{\otimes n}(\xx_1) = (\psi_i(\xx_1), \psi_i(\xx_2), \ldots, \psi_i(\xx_n))$ for every $\xx$. Then, we claim that $\Psi_i$ is a pair of $\epsilon$-Hypercube embedding of the unit ball. To verify this, observe that $| \langle \s{x}{1},\s{x}{2} \rangle - \langle \Psi_1(\s{x}{1}), \Psi_2(\s{x}{2}) \rangle| \leq \sum_{i \leq n} | \s{x}{1}_i\s{x}{2}_i - \langle \psi_1(\s{x}{1}_i),\psi_2(\s{x}{2}_i) \rangle| \leq n \cdot \epsilon/n = \epsilon.$
\end{proof}

\ignore{
The following lemma analyze the effect of discretization in the errors incurred and will be useful in proving \ref{thm:optimal-kernel-ball}.

\begin{lem}
Every function $h \in \C_L(B)$ is $\Theta(BL/\epsilon)$-Lipschitz. 
\end{lem}
\begin{proof}
Every hypothesis in $h \in \C_L(B)$ can be written as $h(x) = \sum_{i \leq m} \alpha_i \cdot K(\s{x}{i}, \xx)$ where $K$ is a \regular $L$-Lipschitz continuous kernel and (using Fact \ref{fact:pegasos-sgd-svm-guarantee}) $\sum_{i\leq m} |\alpha_i| \leq \Theta(B/\epsilon)$. Since each $K$ is $L$-Lipschitz, $h$ is $\Theta(BL/\epsilon)$-Lipschitz.
\end{proof}

\begin{lem} \label{lem:error-of-Lipschitz-functions}
Let $h$ be any $q$-Lipchitz function. Then, $\cL_{\cD}(h) \leq \cL_{\cD^{\Psi}}(h) + q(\epsilon^2/BL)$
\end{lem}
\begin{proof}
It is enough to bound $\|\ell(h(\xx), y)- \ell(h(\Psi(\xx)), y)| \leq |h(\xx)-h(\Psi(\xx)| \leq q \|\xx-\Psi(\xx)\|_2 \leq q \sqrt{(\langle \xx, \Psi(\xx) \rangle^2} \leq q \frac{\epsilon^2}{BL}.$
\end{proof}
}
We can now complete the proof of Theorem \ref{thm:solid}.

\begin{proof}[Proof of Theorem \ref{thm:solid}]
We first describe our algorithm to learn the class of linear classifiers associated with $L$-Lipschitz continuous \regular kernels over the solid cube. 

For every distribution $\cD$ over $[0,1]^n \times \Y$, via the $\frac{\epsilon^2}{100BL}$-hypercube embedding $\Psi_2:[0,1]^n \rightarrow \zo^{nt}.$ , we obtain a distribution $\cD^{\Psi_2}$ over $\zo^{nt} \times \Y$, where $t=\tilde{O}(\frac{n^2}{\epsilon^4}B^2L^2)$. By definition of $\cD^{\Psi_2}$, we can simulate access to i.i.d. samples from $\cD^{\Psi_2}$ given access to i.i.d. samples from $\cD$ and use \ref{thm:mkl} to obtain an efficient algorithm with sample complexity $\tilde{O}(\frac{n^3 B^4L^2}{\epsilon^7})$ to find a hypothesis $h^*$ that has error at most $\opt_{\cD^{\Psi}}(\C_{nt}(B^2/\epsilon)) + \epsilon.$ We will then be done if we can show: \begin{align*}\opt_{\cD}(\Cs_n(B))  \leq \opt_{\cD^{\Psi}}(\C_{nt}(B^2/\epsilon)) + O(\epsilon B^2 L).\end{align*}

Then we get the desired result by taking $\epsilon \to \frac{\epsilon}{B^2 L}$. 
First, using fact (\ref{fact:pegasos-sgd-svm-guarantee}) we know there exists an $\epsilon$-approximate solution $h^*$ such that
\begin{align*}
h^*(\xx) = \sum \alpha_i k(\s{x}{i},\xx), &\quad \|\alpha\|_1\le O(B^2/\epsilon)
\end{align*} 

Note that if $k(\s{x}{1},\s{x}{2})=g(\dotp{\s{x}{1}}{\s{x}{2}})$ is a kernel over $[0,1]^n$ then we can define over the hypercube $\{0,1\}^{nt}$ a \regular kernel:
\[\tilde{k}(\s{\bar{x}}{1},\s{\bar{x}}{2})=g(\frac{1}{t}\dotp{\s{\bar{x}}{1}}{\s{\bar{x}}{2}})).\]
Let $\tilde{h}(\bar{\xx})=\sum \alpha_i \tilde{k}(\Psi_1(\s{x}{i}),\bar{\xx})$. Note that $\frac{1}{t} \left<\Psi_1(\s{x}{1}),\Psi_2(\s{x}{2})\right> < n$, hence we have by $L$-Lipschitness of $g$:
\begin{align*}
\|h^*(\xx) - \tilde{h}(\Psi_2(\xx))\| \le \sum |\alpha_i| |k(\s{x}{i},\xx)-\tilde{k}(\Psi_1(\s{x}{i}),\Psi_2(\xx))|\le O(\epsilon B^2L)
\end{align*}
\ignore{
Using Lemma \ref{lem:error-of-Lipschitz-functions} for $h^{*}$ we obtain that                                                         1) $\cL_{\cD}(h^*) \leq \cL_{\cD^{\Psi}}(h^*) + 3\epsilon$. Using it for every $h \in \C(B)$, we obtain that $\opt_{\cD}(\C(B)) \leq \opt_{\cD^{\Psi}}(\C(B)) + \epsilon.$ Combining the two results we obtain that $\cL_{\cD}(h^*) \leq \opt_{\cD}(\C(B)) + 3 \epsilon.$}\end{proof}
\ignore{
\begin{remark}
While we restricted our attention to strongly \regular kernels, the result may be extended to larger classes of kernels. For example the RBF kernel has the form $k(\xx,\yy)=f(\xx) g(\dotp{\xx}{\yy}) f(\yy)$ (where $f(\xx)=\exp(-\|\xx\|^2/\sigma)$ and $g(\dotp{\xx}{\yy})=\exp(2\dotp{\xx}{\yy}/\sigma)$. We can, for example, discretize the unit--cube and divide the problem into $\mathrm{poly}(n,1/\epsilon)$ learning problems over domains with constant norm, i.e. $f(\xx)=c$. 

 On each learning problem $f(\xx)$ is then constant, hence leads to learning problems with a strongly \regular kernel\footnote{note that if $f(x)g(x,y)f(y)$ is a kernel then for any $\tilde{f}$: $\tilde{f}(x)g(x,y)\tilde{f}(y)$ is a kernel.}.
This result however, would still lead to exponential dependence on the width of the RBF kernel, hence in some sense is weak. It would be interesting to find out if by exploiting new association schemes, one can obtain tighter guarantees for learning \regular kernels, or other rich families, over the unit cube.
\end{remark}
}
\section{Proof of \cref{thm:lower}}\label{prf:lower}
\lowb*
We next set out to show that no \emph{fixed} regular kernel can uniformly approximate conjunctions, this result relies on a similar result by \cite{klivans2007lower}, who showed that there is no linear subspace of dimension $d=2^{o(\sqrt{n})}$ whose linear span can uniformly approximate all conjunctions. Using the Johnson Lindenstrauss style low-dimensional embedding, we prove that an existence of a kernel that uniformly approximates all conjunctions immediately implies a low dimensional RKHS embedding with this property. \cref{thm:lower} then becomes an immediate corollary of \cref{thm:main}.
We let $\con_n=\{ c_{I'}(\xx): c_{I'}(\xx) = \mathop\wedge_{i\in I} \xx_i~ I\subseteq [n]\}$ denote the class of conjunctions over the hypercube $\X_n$. 

Our lower bound works in two steps: First we show that no \emph{fixed} \regular kernel can uniformly approximate conjunctions, this result relies on a similar result by \cite{klivans2007lower}, who showed that there is no linear subspace of dimension $d=2^{o(\sqrt{n})}$ whose linear span can uniformly approximate all conjunctions. Using the Johnson Lindenstrauss style low-dimensional embedding, we prove that an existence of a kernel that uniformly approximates all conjunctions immediately implies a low dimensional RKHS embedding with this property. As a second step we show, using minmax argument and convexity of $F_{S,\lambda}$, that for some distribution, all \regular kernels must fail.

\begin{lem}\label{lem:kernelklivans}
For sufficiently large $n$, there exists a conjunction $c(\xx)\in \con_n$ and a layer $S_{p,n}= \{\xx\in \{0,1\}^n, \sum \xx_i =p \}$ such that for every fixed \regular kernels $k$, if $B_n = 2^{o(\sqrt{n})}$:
\begin{align*}
\min_{\|\ww\|< B_n}\max_{\xx\in S_{p,n}}|c(\xx)- \dotp{\ww}{\phi(\xx)}|> \frac{1}{6}
\end{align*}
\end{lem}
\begin{proof} Assume to the contrary. Fix $p$ and consider $c_{I'}$ a conjunction with $|I|=v$ for some fixed $v\le p$.
We obtain that for all $\|\xx\|=p$, there is some $\|\uu_{I'}\|= 2^{o(\sqrt{n})}$ and $k$, such that:
\begin{align*}
\left|c_{I'}(\xx)- \dotp{\uu}{\phi(\xx)}\right|\le \frac{1}{6}.
\end{align*}
Since $\|\phi(\xx)\|<1$ and $\|\phi(\xx)\|$ depends only on $p$ we can, by choosing $\ww_{I'} = \|\phi(\xx)\|\cdot \uu$, obtain a vector $\ww_{I'}$ such that:
\begin{align*}
\left|c_{I'}(\xx)- \ww_{I'}\cdot \frac{\phi(\xx)}{\|\phi(\xx)\|}\right|\le \frac{1}{6}.
\end{align*}

By the representer theorem, we may assume that $\ww_I= \sum_{\|\s{x}{i}\|=p}\beta_i \phi(\s{x}{i})$ for some $\beta$. 
Since the kernel is \regular, and thus invariant under permutations, one can show that for every conjunction $c_{I}(\xx)$ with $|I|=v$ literals, we have that for some $\ww_{I}$:
\footnote{
Indeed, let $\pi$ be a permutation such that $\pi(I) = I'$. Then, $\ww_{I}=\sum_{\|\s{x}{i}\|=s} \beta_i \phi(\pi_{I,I}(\s{x}{i}))$. Further, $\|\ww_{I'}\|=\|\ww_I\|$ and clearly satisfies \eqref{eq:Iprime}, for all $\|\xx\|=p$.}
\begin{align}\label{eq:Iprime}
\left|c_{I}(\xx)- \ww_{I}\cdot \frac{\phi(\xx)}{\|\phi(\xx)\|}\right|\le \frac{1}{6}.
\end{align}

Next, since $c_{I}(\xx)\in \{-1,1\}$, we can rewrite \eqref{eq:Iprime} as :
\begin{align*}
\frac{5}{6\|\ww_I\|}<\frac{c_I(\xx)\ww_I\cdot \phi(\xx)}{\|\ww_I\|\cdot\|\phi(\xx)\|}<\frac{7}{6\|\ww_I\|}.
\end{align*}
We can apply JL Lemma (see for example (\cite{arriaga1999algorithmic} corollary 2), onto the kernel space, to construct a projection $T :H\to \mathbb{R}^d$ where $d=O(\|\ww\|^2\log 1/(\delta))$ such that w.p $(1-\delta)$, a uniformly random sample from the hypercube will satisfy:
\begin{align*}
\frac{1}{3\|\ww_I\|}<\frac{5}{12\|\ww_I\|}<\frac{c_I(\xx)T(\ww_I)\cdot T(\phi(\xx))}{\|T(\ww_I)\|\cdot\|T(\phi(\xx))\|}<\frac{7}{12\|\ww_I\|}<\frac{4}{3\|\ww_I\|}.
\end{align*}
Choosing $\delta=O(2^{-n})$ and applying union bound over all literals of size $v$, we obtain a subspace $d=O(2^{o(\sqrt{n})}n)$ such that for every $\xx$ in the hypercube.

\begin{align*}
| c_{I}(\xx)- \alpha_I\cdot T(\phi(\xx))| < \frac{1}{3}
\end{align*}
Where $\alpha_I = \|\ww_I\|\frac{T(\ww_I)}{\|T(\ww_I)\|}$.
Next consider the $d$ mappings $g_i(\xx) = \left(T(\phi(\xx))\right)_i$. We've shown that for some linear combination
\begin{align*}
|c_I(\xx) - \sum \alpha_{I,i} g_i(\xx)|<\frac{1}{3}
\end{align*}
Taken together we have shown that for an arbitrary size $p$ and arbitrary number of literals $v$ there exists a set of mapping $g^{(p,v)}_1,\ldots, g^{(p,v)}_d$ with $d=2^{o(\sqrt{n})}$ that can approximate within $\epsilon= \frac{1}{3}$ accuracy each conjunction on samples of size $p$. We can extend each mapping $g^{(p,v)}$ to the whole hypercube by considering
\begin{align*}
g^{(p,v)}(\xx) =\begin{cases} g^{(p,v)}(\xx) & \|\xx\|=p \\ 0 & \mathrm{o.w}\end{cases}
\end{align*}
Thus, taking a union of all $g^{(p,v)}$ we obtain a set of $O(n^2 2^{o(\sqrt{n})})$ mappings that can approximate each conjunction, uniformly over the hypercube. This contradicts the result of \cite{klivans2007lower} such that for every $2^{o(\sqrt{n})}$ dimensional subspace $V$, there's some conjunction which cannot be approximated by any element of $V$.
\end{proof}

Applying a minmax argument we can restate the result as follows

\begin{lem}\label{lem:lb}
For every fixed \regular kernel $k$, there exists a distribution $D$ over $\X_n$ and a conjunction $c(\xx)$ so that:
\begin{align*}\min_{\|\ww\|<B} \EE{|c(\xx)-\dotp{\ww}{\phi(\xx)}|}<\frac{1}{12}\end{align*} 
then $B=2^{(\Omega(\sqrt{n}))}$.
\end{lem}
\begin{proof}
Indeed, the negation of the statement would yield that letting $\mathcal{D}$ be the family of all distributions over $\X_n$, then:

\begin{align*}\max_{D\sim \mathcal{D}}\min_{\|\ww\|<B} \mathbb{E}_{\xx\sim D}|c(\xx)-\dotp{\ww}{\phi(\xx)}|<\frac{1}{12}\end{align*} 
Exploiting the convexity of the objective in terms of $\ww$ and $D$ we can apply the minimax principle and obtain a contradiction to \cref{lem:kernelklivans}.
\end{proof}
\subsection{Putting it all together}
The proof is an immediate corollary of \cref{lem:lb} and the existence of a universal kernel as presented in \cref{thm:main}
\subsection{Learning Conjunctions via \regular kernels}\label{sec:upper-bound}
Given our lower bound for learning conjunctions through kernels, the first natural question is whether the upper bound $2^{\tilde{O}(\sqrt{n} \log{(1/\epsilon)})}$ is attainable using \regular kernel methods. The $L_1$ regression algorithm introduced in \cite{kalai2008agnostically} employs an observation of \cite{paturi1992degree} that conjunctions can be approximated in monomial space of degree $\tilde{O}(\sqrt{n} \log{(1/\epsilon)})$ to learn in time $2^{\tilde{O}(\sqrt{n} \log{(1/\epsilon)})}$. They also make the observation, that the algorithm may be implemented by an SVM-like convex formulation -- however their analysis relies on the dimension of the linear classifier being small. We show that using a slightly modified version of the polynomial kernel, standard SVM analysis can achieve the same learnability result. Such an analysis implies, in particular, thatwill succeed in achieving the same performance. 

We use a similar analysis to show an improved bound under distributional assumptions. We begin by stating the main fact exploited by all algorithms for learning conjunctions

\begin{fact}\cite{paturi1992degree}\label{fact:paturi}
For every conjunction $c(\xx)$ over the hypercube $\X_n$ there exists a polynomial $p_{I}(\xx)= \sum_{I\subseteq \{0,1\}^n} \alpha_{I} \prod_{i\in I} \s{x}{i}$ of degree $O(n^{\sqrt{n}\log 1/\epsilon})$. whose coefficient  satisfy $\sum \alpha_{I}^2 = 2^{\tilde{O}(\sqrt{n} \log{(1/\epsilon)})}$.
\end{fact}

\begin{theorem}\label{thm:cupper}
For every layer of the hypercube $S_{p,n}$, 
There is a \regular kernel $k$ and an embedding $\phi :S_{p,n} \to H$ such that for every conjunction $c_{I}(\xx)$ there is $\|\ww\| = 2^{\tilde{O}(\sqrt{n} \log{(1/\epsilon)})}$ such that
\begin{align*}
|c_{I}(\xx)- \dotp{\ww}{\phi(\xx)}| < \epsilon
\end{align*}
\end{theorem}
\begin{proof}
Our choice of kernel is inspired by the basis kernels of the Johnson Scheme. Namely, set $T_n=O(\sqrt{n}\log 1/\epsilon)$. we choose as kernel
\begin{align*}
k(\s{x}{i}\cdot \s{x}{j}) = \frac{1}{N_p}\cdot \sum_{t\le T_n} {{\s{x}{i}\cdot \s{x}{j}}\choose t}
\end{align*}
where $N_p = \sum_{t\le T_n} {p\choose t} = O(n^{\sqrt{n}\log 1/\epsilon})$.
One can show that for any two points $\s{x}{i}$ and $\s{x}{j}$

\begin{align*}
k(\s{x}{i}\cdot \s{x}{j}) = \frac{1}{N_p} \sum_{|I|\le T_{n}} \prod_{k\in I} \s{x}{i}_k\cdot \s{x}{j}_k
\end{align*}

Let $H$ be the associated Hilbert space with the kernel $k$, then one can observe that the kernel $k$ embeds the sample points in the space of monomials together with the standard scalar product normalized by $\frac{1}{N_p}$. by fact \ref{fact:paturi}, we know that there exists $p\in H$ whose $\ell_2$ norm over the coefficient is at most $2^{\tilde{O}(\sqrt{n} \log{(1/\epsilon)})}$. which in turns implies that $\|p\|^2_{H} = \frac{1}{N_p} |\sum \alpha_{I}^2| = 2^{\tilde{O}(\sqrt{n} \log{(1/\epsilon)})}$.
\end{proof}
\ignore{
\paragraph{Learning Conjunctions under sparseness assumptions}
By \thmref{thm:lower} the optimal achievable rate for learning conjunctions via \regular kernel is attained by the result above. In what follows we consider the case where the distribution is supported on $s$ sparse vector, namely $P(\sum x_i \le s )=1$ for some $s \ll n$. For simplicity, we will assume that $P(\sum x_i = s )=1$. Under the sparseness assumptions, the above approach can be simplified to using degree at most $s$ polynomials, thus leading to a running time of $O(n^{s})$ (which will also be achieved by the aforementioned kernel methods). However, as we next state, under the sparseness assumption there is a kernel which can achieve margin of $O(2^{s})$ (instead of $\approx n^{s}$). Since such a kernel exists \algname will enjoy the improved rate of $O(2^{s})$ over existing methods. This illustrates that even though in worst case we should not expect \algname to perform better then known approach, it could improve over existing result in a task specific manner.
\begin{theorem}\label{thm:upper}
Let $H$ be the hypothesis class of conjunctions and assume $\P$ is a distribution supported on $s$-sparse vectors i.e. $\P(\sum \s{x}{i} = s) =1$. Then for every $c_{I}\in \conj$ there exist $f_{H,\ww}\in \C(O(2^{s}))$ such that a.s.
\begin{align*}
\dotp{\ww}{\phi_{H}(\xx)}=c_I(\xx).
\end{align*}
As a corollary, we can learn the class in time $\mathrm{poly}(2^s,\frac{1}{\epsilon}, n)$.
\end{theorem}

For a conjunction $c_I$ with $\ell$ literals, we let $k_{\ell}$ be defined by
\[ k(\s{x}{i},\s{x}{j}) = {{s}\choose{\ell}}^{-1}\cdot {{\s{x}{i}\cdot \s{x}{j}}\choose{\ell}}.\]
To show that the kernel learns the class of conjunctions with $\ell$ literals in time $\mathrm{poly}(2^s,\frac{1}{\epsilon})$ it would be enough to show that if $H$ is the associated Hilbert space then there is $\ww\in H$ such that $\|\ww\|\le 2^s$ and
\begin{align*}
\dotp{\ww}{\phi(\s{x}{i})} = c_{I}.
\end{align*}
We can then employ surrogate loss function) to learn the class (since a surrogate convex loss function upper bounds the zero one error). Indeed, given the conjunction $c_{I}$ let $\mathbf{c}$ be a vector such that $\mathbf{c}_i = 1$ for all $i\in I$ and $\mathbf{c}_i=0$ else. If we let $\ww={s\choose \ell}\phi(\mathbf{c})$ then one can verify that:
\begin{align*}
\dotp{\ww}{\phi(\xx)} = {{\mathbf{c}\cdot \xx}\choose{\ell}} = 
\begin{cases}
1 &  \sum_{i\in I} \s{x}{i} =\ell \\ 
0 & \mathrm{else}  
\end{cases} = c_{I}(\xx).
\end{align*}
The element $\phi(\mathbf{c})$ belongs to $H$ and we have that

\begin{align*}
\|{s\choose \ell} \phi(\mathbf{c})\|^2 = {s\choose \ell}^2{s\choose \ell}^{-1} {{\mathbf{c}\cdot \mathbf{c}}\choose{\ell}}= {s\choose \ell} = O(2^s).
\end{align*}

\ignore{
\begin{proof}
Fix a conjunction $c_I$ and let $k=|I|$. By \cite{o2003new}, for some $\tau>2$ we have that if  $S(\xx)=R(X)^{\log 1/\epsilon}$ and 
\begin{align*}
R(\xx)= \left(\frac{1}{\tau} C_{\sqrt{k}}\left(\frac{\sum_{i\in I} \s{x}{i} + k\epsilon}{k-1+2k\epsilon}\right) \right),
\end{align*}
where $C_{\sqrt{k}}$ is the $\sqrt{k}$--Chebyshev polynomial of the first kind, then: 
\begin{align*}
| S(\xx) - c_{I}(\xx)| < \epsilon,
\end{align*}
Given $S$, first we will show that there exists a \regular kernel $k_1$, embedding $\phi_1: \X \to H_1$ and $\|\ww_1\|_{H_1}=O\left(2^{\sqrt{n}}\right)$ such that $\dotp{\ww_1}{\phi_1(\xx)}=R(\xx)$. 

For that we consider the kernel $k(x_i,x_j)= \frac{1}{1-\frac{1}{2 n}\s{x}{i}\cdot \s{x}{j}}$ introduced in \cite{shalev2011learning}. Let $H_1$ be the Hilbert space associated with $k_1$. Consider the polynomial $p(x) = \frac{1}{\tau} C_{\sqrt{k}}(a)$. 

By Lemma 1. in \cite{shalev2011learning}, if $p(a)=\sum \beta_j a^j$ and $\sum \beta_j^2 (2n)^j <B$ then for every $\|\uu\|_2\le 2$ there exists $\|\ww_1\|_H = O(B)$ such that $\dotp{\ww}{\phi(\xx)} = p(\uu\cdot \xx )$.\footnote{ \cite{shalev2011learning}, state the result for the case we choose the hyper-parameter $\nu=\frac{1}{2}$ and have the kernel $k(x_i,x_j)=\frac{1}{1-\frac{1}{2} x_i\cdot x_j}$, in that case, the result will hold with  $\sum \beta_j^2 2^j <B$, we in contrast choose $\nu=\frac{1}{2n}$ hence the difference.}  
For the Chebyshev's polynomial coefficients we have (again by \cite{shalev2011learning}, Lemma 6) $|t_{\sqrt{n},j}|\le \frac{e^{\sqrt{n}+j}}{\sqrt{2\pi}}$. Taken together we have that
the coefficient of $p$ equal $\beta_j \le \frac{t_{\sqrt{n},j}}{\tau}$ and we have:

\begin{align*}
\sum_{j\le \sqrt{n}} \beta^2_j (2n)^j< \frac{1}{\tau^2}\sum \frac{1}{2\pi}e^{2\sqrt{n}+j(+\log 2n)} = \tilde{O}(2^{\sqrt{n}})
\end{align*}
\end{proof}

Embedding $\xx \to (\xx,1) \in \X_{n+1}$ we can see that $R(\xx) = p( \uu\cdot (\xx,1))$ where $\uu_{i}= \frac{1}{k-1+2k\epsilon}$ for every $i\in I$ and $\uu_{n+1} = \frac{k\epsilon}{k-1+2k\epsilon}$.
Thus letting $\phi_1: \X \to H$ be the embedding defined by the kernel:
 \begin{align*} \phi_1(\s{x}{i})\cdot \phi_1(\s{x}{j}) =k_1(\s{x}{i},\s{x}{j}) = \frac{1}{1- \frac{1}{2n} (\s{x}{i}\cdot \s{x}{j} +1)}\end{align*}
We have that for the kernel $k_1$ and the embedding $\phi_1$, for some $\ww_1\in H_1$ with $\|\ww_1\|=\tilde{O}(2^{\sqrt{n}})$ we have that $\dotp{\phi_1(\xx)}{\ww_1}_{H_1} = R(\xx)$.

We next embed $H_1$ into a new Hilbert space $H_2$ through a polynomial kernel. 
Specifically, consider the embedding $\phi_2 : H_1\to H_2$ induced by a \regular kernel $k_2(\uu_i,\uu_j) = (\dotp{\uu_i}{\uu_j}_{H_1})^{\log 1/\epsilon}$, where $\uu_i,\uu_j \in H_1$.

It is easy to see that for every $\xx \in \X$ we have that $\dotp{\phi_2(\ww_1)}{\phi_2(\phi_1(\xx))}_{H_2} = S(\xx)$, and further $\|\phi_2(\ww_1)\|^2_{H_{2}}= (\|\ww_1\|^2_{H_1})^{\log 1/\epsilon} = O(2^{2 \sqrt{n}\log 1/\epsilon})$. Overall we have $\ww_2 = \phi_2(\ww_1)$ such that $\|\ww_2\|_{H_2}= \tilde{O}(2^{\sqrt{n}\log 1/\epsilon})$ and 
\begin{align*} \dotp{\ww_2}{\phi_2(\phi_1(\xx))}_{H_2} = 
\dotp{\phi_2(\ww_1)}{\phi_2(\phi_1(\xx))}_{H_2}=
\left(\dotp{\ww_1}{\phi_1(\xx)}_{H_1}\right)^{\log 1/\epsilon} = (R(\xx))^{\log 1/\epsilon}=
S(\xx). 
\end{align*} It is easy to verify that the embedding $\phi_2\circ \phi_1$ is induced by the kernel $k(x_i,x_j) = k_2( k_1 (\s{x}{1}\cdot \s{x}{2}))$, hence it is a \regular kernel.}}

\end{document}